%% file: main.tex
\documentclass{article}

\PassOptionsToPackage{numbers}{natbib}
\usepackage[preprint]{neurips_2025}
\usepackage[utf8]{inputenc} 
\usepackage[T1]{fontenc}    
\usepackage{amsmath}
\usepackage{amsfonts}
\usepackage{amsthm}
\usepackage{hyperref}       
\usepackage{url}            
\usepackage{booktabs}       
\usepackage{amsfonts}       
\usepackage{nicefrac}       
\usepackage{microtype}      
\usepackage{multirow}       
\usepackage{xcolor}         
\usepackage{graphicx}       
\usepackage{subfigure}

\usepackage{enumitem}
\usepackage{comment}
\usepackage{amssymb}
\usepackage{algorithm}
\usepackage{algorithmic}
\usepackage{wrapfig}
\usepackage{dirtytalk}

\usepackage{caption}
\usepackage{bm}
\usepackage{diagbox}

\newcommand{\op}[1]{\operatorname{#1}}
\newtheorem{theorem}{Theorem}[section]   
\newtheorem{definition}{Definition}[section]  %
\newtheorem{assumption}{Assumption}[section]  

\title{Dispelling the Curse of Singularities in \\Neural Network Optimizations}

\author{%
Hengjie Cao$^{1}$  \hspace{1.0em}
Mengyi Chen$^{1,}$\thanks{Co-first authors.} \hspace{1.0em}
Yifeng Yang$^{1,\ast}$\hspace{1.0em}
Fang Dong$^{1}$  \hspace{1.0em}
Ruijun Huang$^{1}$ \hspace{1.0em}
\\
\textbf{Anrui Chen}$^{1}$ \hspace{1.0em}
\textbf{Jixian Zhou}$^{1}$ \hspace{1.0em}
\textbf{Mingzhi Dong}$^{2}$\hspace{-0.1em} \hspace{1.0em} 
\textbf{Yujiang Wang}$^{3}$ \hspace{1.0em}
\textbf{Dongsheng Li}$^{1}$ \hspace{1.0em}
\\
\textbf{Wenyi Fang}$^{5}$\hspace{1.0em}
\textbf{Yuanyi Lin}$^{5}$\hspace{1.0em}
\textbf{FAN WU}$^{5}$\hspace{1.0em} 
\textbf{Li Shang}$^{1,\dag}$ \vspace{0.5em}
\\
$^1$ Fudan University \hspace{1.0em}
$^2$ University of Bath  \hspace{1.0em}  
$^3$ Oxford Suzhou Centre for Advanced Research \hspace{1.0em}\\
$^5$ Huawei \hspace{1.0em}
}

\begin{document}

\maketitle

\begin{abstract} 
\label{sec:abstract}

This work investigates the optimization instability of deep neural networks from a less-explored yet insightful perspective: the emergence and amplification of singularities in the parametric space. Our analysis reveals that parametric singularities inevitably grow with gradient updates and further intensify alignment with representations, leading to increased singularities in the representation space. We show that the gradient Frobenius norms are bounded by the top singular values of the weight matrices, and as training progresses, the mutually reinforcing growth of weight and representation singularities, termed the curse of singularities, relaxes these bounds, escalating the risk of sharp loss explosions. To counter this, we propose Parametric Singularity Smoothing (PSS), a lightweight, flexible, and effective method for smoothing the singular spectra of weight matrices. Extensive experiments across diverse datasets, architectures, and optimizers demonstrate that PSS mitigates instability, restores trainability even after failure, and improves both training efficiency and generalization.

\end{abstract}

\input{sections/1_intro}

\input{sections/2_analysis}

\input{sections/3_method}

\input{sections/4_experi}

\input{sections/5_related_work}

\input{sections/6_conclusion}

\newpage

\input{main.bbl}


\newpage
\section*{Appendix}
\appendix
\input{sections/appendix}

\newpage
\section*{NeurIPS Paper Checklist}
\input{sections/checklist}

\bibliographystyle{plain}

\end{document}

%% file: sections/1_intro.tex
\section{Introduction}

Training instability, often manifesting as sudden loss explosions or erratic fluctuations, remains a long-standing and critical challenge in deep neural network (DNN) optimization \citep{philipp2017exploding, chowdhery2023palm}. These instabilities frequently arise without early warning signals and can occur at any stage of training \citep{kanai2017preventing, ribeiro2020beyond}, often leading to catastrophic collapse of the learning process. As a result, developing a stable training configuration is labor-intensive and error-prone: even small modifications in model architecture, optimizer, or data can undermine trainability and necessitate costly re-training.
This problem is especially severe in large-scale foundation models~\citep{zhang2022opt}, where each training run consumes weeks of computational time on thousands of GPUs. In these contexts, repeated instability-driven failures  hinder iterative development, making a stable and restart-free training approach crucial.


Existing methods to address instability are largely empirical, e.g., gradient clipping \citep{Clippy2023, molybog2023theory}, and are sensitive to hyperparameter choices, often either failing to prevent divergence or excessively constraining learning dynamics \citep{Zhang2020Why, NEURIPS2020_9ecff545}. While theoretical insights link instabilities to sharp loss landscapes characterized by large Hessian eigenvalues \citep{gilmer2021loss, kim2023stability}, the high complexity of loss landscape limits previous analyses from fully  revealing the underlying mechanism of training instability.


In this work, we analyze training instability through the lens of growing singularities in the parametric space, a less-investigated perspective. Empirical investigations reveal that, across layers, the weight matrices, such as the query and key matrices in Transformers \citep{vaswani2017attention}, become increasingly singular as training progresses, consistent with recent findings \citep{mousavi2022neural, zangrando2024neural}. More strikingly, we find that imminent catastrophic loss explosions are preceded by rapidly escalating singularities in both the parameter space and the representation space, characterized by growing rank deficiencies \citep{dong2021attention, noci2022signal}.

These phenomena exhibit a mutually reinforcing trend: growing singularities in weight matrices lead to more degenerate representations, which in turn reinforce the singularity of subsequent gradient updates. This feedback loop culminates in a self-amplifying mechanism that triggers sharp, unrecoverable loss spikes. We term this failure mode \textit{the curse of singularities}, a fundamental yet previously underexplored cause of optimization instability in DNNs.


To substantiate these observations, we adopt a simplified one-layer Transformer model and provide analysis grounded in the QK-gradient approximation \citep{bao2024self}. We prove that the stable rank \citep{rudelson2007sampling} (a soft surrogate for matrix rank) of weight matrices is provably lower-bounded to increase following each gradient update. Furthermore, we show that alignment between weight and representation singular vectors becomes tighter after backpropagation, further increasing representational singularity and reinforcing the instability. Crucially, we demonstrate that the Frobenius norm of gradients is bounded by the top singular value of the weight matrices, which means as the network becomes more singular, it also becomes more susceptible to large, unstable gradient updates. This dynamic relaxation of gradient norm bounds provides a mechanistic link between the growing singularities and sharp loss landscapes observed in prior work~\cite{gilmer2021loss,kim2023stability}. 


Motivated by these findings, we introduce \textit{Parametric Singularity Smoothing (PSS)}, a lightweight and adaptable technique that regularizes the singular value spectra of weight matrices during training. PSS proactively mitigates the mutual reinforcement of singularities, effectively restoring trainability even after instability onset. Empirical evaluations across diverse models, datasets, and optimization methods demonstrate that PSS consistently prevents training collapses, improves convergence efficiency, enhances generalization performance, and can restore trainability even after the occurrence of perceivable instabilities.

Our contributions are summarized as follows:
\begin{itemize}[leftmargin=*]
\item We identify a mutually reinforcing growth of singularities in weight and representation matrices, termed the \textit{curse of singularities}, as a primary cause of training instability in deep neural networks.
\item We provide analysis on a simplified Transformer model, showing how gradient updates amplify parametric singularities and destabilize training.
\item We propose Parametric Singularity Smoothing, a lightweight and effective method that mitigates singularity buildup and prevents loss explosions across diverse models and tasks.
\end{itemize}

%% file: sections/2_analysis.tex
\section{Analysis}
\label{analysis}
We begin with preliminary knowledge and notations. Next, we empirically demonstrate rank co-collapse in representation and parametric spaces, termed the curse of singularities, and explore its mechanism using a simplified one-layer transformer, providing key insights. Finally, we link singularity amplification to training instability.

\subsection{Preliminaries}
\label{anslysis:preliminaries}
This section outlines key definitions, assumptions, and a one-layer transformer framework, forming the foundation for our subsequent analysis. 

\subsubsection{Definitions and Assumptions}

We introduce the Stable Rank (SR) \citep{rudelson2007sampling} to quantify parametric singularities. For a weight matrix \(\mathbf{W} \in \mathbb{R}^{m \times d}\), we perform Singular Value Decomposition (SVD) to obtain singular values \(\{\sigma_i\}_{i=1}^{\min(m,d)}\), left singular vectors \(\{\mathbf{u}_i\}_{i=1}^m \in \mathbb{R}^m\), and right singular vectors \(\{\mathbf{v}_i\}_{i=1}^d \in \mathbb{R}^d\), such that \(\mathbf{W} = \sum_{i=1}^{\min(m,d)} \sigma_i \mathbf{u}_i \mathbf{v}_i^\top\). Throughout the paper, we assume all singular values are sorted in descending order, namely $\sigma_1 \geq \sigma_2 \geq \dots \sigma_{r}>0$, and we denote the $i$-th singular value of $\mathbf{W}$ as $\sigma_i(\mathbf{W})$.

\begin{definition}[Parametric Singularity]
The parametric singularity of a matrix \(\mathbf{W}\) is its stable rank, defined as the ratio of the squared Frobenius norm \(\|\mathbf{W}\|_F^2\) to the squared spectral norm \(\|\mathbf{W}\|_2^2\):
\begin{equation}
    \operatorname{SR}(\mathbf{W}) = \frac{\|\mathbf{W}\|_F^2}{\|\mathbf{W}\|_2^2} = \frac{\sum_{i=1}^{r} \sigma_i^2}{\sigma_1^2}.
\end{equation}
\end{definition}
SR provides a smoothed measure of matrix rank, reflecting weight matrix singularity. A low SR indicates a pronounced singularity, where a few singular vectors dominate the parametric space. 

Similarly, we define the singularity in the representation space. Let \(\mathbf{X} = [\mathbf{x}_1, \mathbf{x}_2, \dots, \mathbf{x}_T] \in \mathbb{R}^{T \times d}\) represent an input of \(T\) tokens, each a \(d\)-dimensional vector (\(T > d\)). The expected input representation is \(\mathbf{Z} = \mathbb{E}_{\mathcal{D}} [\mathbf{x}_1, \mathbf{x}_2, \dots, \mathbf{x}_T] \in \mathbb{R}^{T \times d}\), where the expectation is over the dataset \(\mathcal{D}\).

\begin{definition}[Representation Singularity]
For the SVD of \(\mathbf{Z} = \sum_{i=1}^{d} \mu_i \boldsymbol{\alpha}_i \boldsymbol{\beta}_i^\top\), with singular values \(\{\mu_i\}\), left singular vectors \(\{\boldsymbol{\alpha}_i\}\), and right singular vectors \(\{\boldsymbol{\beta}_i\}\), the representation singularity is:
\begin{equation}
    \operatorname{SR}(\mathbf{Z}) = \frac{\|\mathbf{Z}\|_F^2}{\|\mathbf{Z}\|_2^2} = \frac{\sum_{i=1}^{r} \mu_i^2}{\mu_1^2}.
\end{equation}
\end{definition}

Owing to the long-tail feature distribution in the training set \cite{RedditCommentsDataset}, the representation space exhibits significant singularity at the start of training. This is evidenced by an initial \(\operatorname{SR}(\mathbf{Z})\) below 20, as shown in Figure~\ref{figure:metrices}(b), and by the singular value distribution of the representation matrix, where the leading singular value \(\mu_1\) dominates the spectrum, as detailed in Appendix~\ref{appendix:representation-dom}.



To investigate the relationship between singularities in parametric and representation spaces, we quantify their alignment by measuring the cosine similarity between the top singular vectors, which represent the dominant directions due to pronounced singularities in both spaces.
\begin{definition}[Singularity Alignment]
For matrices \(\mathbf{W} \in \mathbb{R}^{m \times d}\) and \(\mathbf{Z} \in \mathbb{R}^{T \times d}\), with top right singular vectors \(\mathbf{v}_1\) and \(\boldsymbol{\beta}_1\), respectively, the singularity alignment \(\phi\) is defined as:
\begin{equation}
    \phi = \langle \mathbf{v}_1, \boldsymbol{\beta}_1 \rangle.
\end{equation}
\end{definition}
Empirical results in Fig.~\ref{figure:metrices} show \(\phi\) ranges between 0.6 and 1.0. For analytical simplicity, we assume \(\phi\) approaches 1, with Appendix~\ref{appendix:empirical_validation} confirming our theorems hold for empirical values.

\begin{assumption}[Strong Singularity Alignment]
The top singular vectors are strongly aligned, satisfying \(\mathbf{v}_1^\top \boldsymbol{\beta}_1 = \sqrt{1 - \epsilon}\), \(\max_{r \geq 2} |\mathbf{v}_1^\top \boldsymbol{\beta}_r| \leq \sqrt{\epsilon}\), and \(\epsilon \to 0\).
\end{assumption}

In the expected representation space \(\mathbf{Z} = \mathbb{E}[\mathbf{X}]\), the right singular vectors \(\{\boldsymbol{\beta}_t\}_{t=1}^d\) of \(\mathbf{Z}\) form semantic bases, with singular values \(\{\mu_t\}_{t=1}^d\) reflecting their prominence. To model token representations in this singular space, we introduce the following assumption:
\begin{assumption}[Semantic Basis Composition]
Each token representation \(\mathbf{x}_i \in \mathbf{X}\) is a linear combination of these bases, \(\mathbf{x}_i = \sum_{t=1}^d c_{i,t} \boldsymbol{\beta}_t\), where the coefficients satisfy \(\mathbb{E}[c_{i,t}] = \mu_t\).
\end{assumption}

\subsubsection{One-Layer Transformer and Gradient Approximation for the \( QK \) Matrix}
Following the configuration outlined in \cite{bao2024self}, we adopt a one-layer Transformer, without loss of generality, to derive the gradient of the query-key parameter matrix for our analysis of the curse of singularities.

\textbf{Softmax Approximation}.\hspace{0.03em}
To facilitate gradient derivation, we approximate the softmax function using a piecewise linear approach \cite{bao2024self}. For a \( T \)-dimensional Key-Query inner product input \( \boldsymbol{\omega} := \frac{\mathbf{X} \mathbf{W}_{QK} \mathbf{x}_T^\top} {\sqrt{d}}\in \mathbb{R}^T \), where \( \mathbf{X} \in \mathbb{R}^{T \times d} \) is a sequence input with \( T \) tokens and $\mathbf{W}_{QK}$ is query-key parameter matrix, the softmax is defined as \( S(\boldsymbol{\omega})_i := \frac{\exp(\omega_i)}{\sum_{j=1}^{T} \exp(\omega_j)} \) for \( i \in [T] \), mapping the i-th element of $\boldsymbol{\omega}$, denoted as $\omega_i$, to the \([0, 1]\) interval. We employ a Taylor expansion around the origin, with coefficients \( \boldsymbol{\gamma}^i := \nabla_i S(\mathbf{0}) = \frac{1}{T} \mathbf{e}^i - \frac{1}{T^2} \mathbf{1} \) and \( \gamma^i_0 := S(\mathbf{0})_i = \frac{1}{T} \), where \( \mathbf{e}^i \) is the \( i \)-th one-hot vector and \( \mathbf{1} \) is the all-ones vector. Thus, \( \gamma^i_i = \frac{1}{T} - \frac{1}{T^2} \) and \( \gamma^i_a = -\frac{1}{T^2} \) for \( a \in [T] \setminus \{i\} \). To stabilize the approximation, we apply a truncation strategy with a fixed threshold \( c \). When the absolute value of the approximation term exceeds \( c \), the slope is clipped to 0:
\[
\begin{aligned}
\tilde{\gamma}_i^i &= 
\begin{cases}
\frac{T-1}{T^2}, & \text{if } \gamma_i^i \omega_i + \frac{1}{T} \gamma_0^i \in [-c, c] \\
0, & \text{otherwise}
\end{cases} \quad
\tilde{\gamma}_a^i &= 
\begin{cases}
-\frac{1}{T^2}, & \text{if } \gamma_a^i \omega_a + \frac{1}{T} \gamma_0^i \in [-c, c] \\
0, & \text{otherwise } 
\end{cases}
\end{aligned}
\]

The resulting approximation in vector form is: \( S(\boldsymbol{\omega}) \approx \tilde{S}(\boldsymbol{\omega}) = \boldsymbol{\Gamma}^\top \boldsymbol{\omega} + \tilde{\boldsymbol{\gamma}}_0,\) where \( \boldsymbol{\Gamma} := [\tilde{\boldsymbol{\gamma}}^1, \tilde{\boldsymbol{\gamma}}^2, \dots, \tilde{\boldsymbol{\gamma}}^T] \) and \( \tilde{\boldsymbol{\gamma}}_0 = [\tilde{\gamma}^1_0, \tilde{\gamma}^2_0, \dots, \tilde{\gamma}^T_0]^\top \).

\textbf{Model and Training Task}.\hspace{0.03em}
Next, we consider a one-layer transformer, simplifying the feed-forward network by assuming an identity activation function and omitting layer normalization.
With \( \mathbf{X} \in \mathbb{R}^{T \times d} \) as a sequence input with \( T \) tokens. 
Using the approximated softmax \( \tilde{S} \), the output of a transformer with an attention layer $A$ and a feedforward layer $F$ is given by:
\[
{F}(A(\mathbf{X}))_T = \mathbf{W}_F \left( \mathbf{W}_V \mathbf{X}^\top \tilde{\mathbf{S}}(\boldsymbol{\omega}) + \mathbf{x}_T \right)  = \mathbf{W}_F \left( \mathbf{W}_V \mathbf{X}^\top \mathbf{\Gamma}^\top \boldsymbol{\omega} + \mathbf{W}_V \mathbf{X}^\top \tilde{\boldsymbol{\gamma}}_0 + \mathbf{x}_T \right),
\]
where \( \mathbf{W}_F := \mathbf{W}_{F_2} \mathbf{W}_{F_1} + \mathbf{I} \), with \( \mathbf{W}_{F_1}, \mathbf{W}_{F_2} \in \mathbb{R}^{d \times d} \) as feed-forward parameters. The query-key parameter matrix \( \mathbf{W}_{QK} := \mathbf{W}_Q^\top \mathbf{W}_K \in \mathbb{R}^{d \times d} \) is optimized jointly, with \(\mathbf{W}_Q\) and \(\mathbf{W}_K\) constrained to be identical throughout training, ensuring \(\mathbf{W}_{QK}\) is treated as a unified parameter. The training task is causal language modeling, predicting the \( (T+1) \)-th token \( \mathbf{y} := \mathbf{x}_{T+1} \in \mathbb{R}^d \) given \( T \) contextual tokens \( \mathbf{X} \). The objective, using squared loss, is:
\begin{equation*}
{\mathcal{J}}(\mathbf{\Theta}) := \frac{1}{2} \mathbb{E} \| \mathbf{y} - {F}(A(\mathbf{X}))_T \|^2,
\end{equation*}
where \( \mathbf{\Theta} \) denotes model parameters, and the expectation is over the dataset.

\textbf{QK-Gradient Approximation}.\hspace{0.03em}
For a sufficiently large sequence length \( T \), the simplified gradient of the objective with respect to \( \mathbf{W}_{QK} \)  \cite{bao2024self} is:
\begin{equation}
\label{equation:origin_gradient_form}
\frac{1}{d} \sum_{i,j,a,b \in [T]} \mathbb{E} \left[ \tilde{\gamma}^i_a \tilde{\gamma}^j_b {P}_{ij} (\mathbf{x}_b^\top \mathbf{W}_{QK} \mathbf{x}_T) \mathbf{x}_a \mathbf{x}_T^\top \right],
\end{equation}
where ${P}_{ij} := (\mathbf{x}_i^\top \mathbf{W}_V^\top \mathbf{W}_F^\top \mathbf{W}_F \mathbf{W}_V \mathbf{x}_j)$ is independent of $\mathbf{W}_{QK}$, under the assumption that the other model parameters do not influence the analysis of the singularity of $\mathbf{W}_{QK}$.


\begin{figure*}[h]
  \centering
    \includegraphics[width=\textwidth]{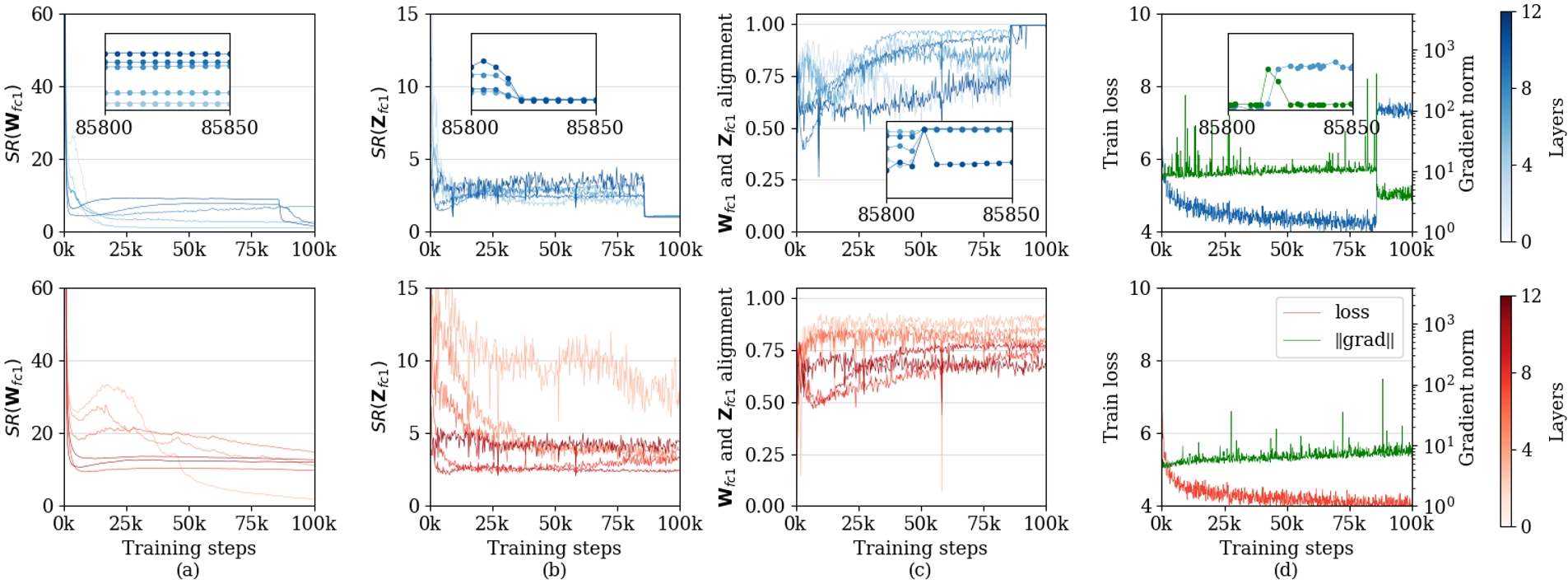}   
  \caption{Evolution of \( \operatorname{SR}(\mathbf{W}) \), \( \operatorname{SR}(\mathbf{Z}) \), singularity alignment \( \phi \), and combined training loss (blue) and gradient norm (green, log scale) in the feedforward module (FFN) across 12 layers. The top panel shows unstable training (LR=2e-4), with a loss spike and sharp SR drop near step 85800, linking singularities to instability; the bottom panel shows stable training (LR=1e-4).}
  \label{figure:metrices}
\end{figure*}

\subsection{Unveiling the Curse of Singularities}
Empirical results in Fig.~\ref{figure:metrices}(a)(b) illustrate the evolution of parametric singularity \( \operatorname{SR}(\mathbf{W}) \) and representation singularity \( \operatorname{SR}(\mathbf{Z}) \) during training. Both metrics exhibit a sharp decline early in training, stabilizing at low values with minor fluctuations thereafter. This trend is consistent across all network components, with additional module results provided in Appendix~\ref{appendix:empirical observation}. The singularity alignment \( \phi \), shown in Fig.~\ref{figure:metrices}(c), sharply increases early in training and remains high, indicating strong alignment between parametric and representation spaces. This alignment intensifies the mutual reinforcement of parametric and representation singularities, exacerbated by suboptimal hyperparameter settings like large learning rates, resulting in \textit{the curse of singularities}.

Building on Section~\ref{anslysis:preliminaries}, we analyze a single-layer transformer under the assumptions to elucidate the mechanism driving the rank co-collapse phenomenon. We first demonstrate that parametric singularities inevitably increase after gradient updates.

\begin{theorem}[Amplification of Parametric Singularity]
\label{theorem:parametric}
The gradient of the loss \(\mathcal{J}\) with respect to \(\mathbf{W}_{QK}\) amplifies parametric singularities, approximated as:
\begin{equation}
\label{eq:qkgrad-simple}
    \nabla_{\mathbf{W}_{QK}} \mathcal{J} = {P}  \mathcal{O}(\sigma_1^2\mu_1^2\phi^2) \sum_{t=1} \mu_t^2 \boldsymbol{\beta}_t \boldsymbol{\beta}_t^\top,
\end{equation}
where \( P = \frac{1}{d} \sum_{i,j,a,b \in [T]} \mathbb{E} \left[ \tilde{\gamma}^i_a \tilde{\gamma}^j_b {P}_{ij} \right] \) is negative and $\mathcal{O}$ denotes the asymptotic order of the term. For large \( T \) and \(\operatorname{SR}(\mathbf{W}_K) > 1 + \frac{\operatorname{SR}(\mathbf{\mathbf{Z}}) - 1}{\phi^2}\), a gradient update with learning rate \(\eta\) reduces the stable rank of \(\mathbf{W}_Q\) and \(\mathbf{W}_K\) by:
\begin{equation}
\label{eq:srw-decrease}
    \triangle \operatorname{SR}(\mathbf{W}_Q) = \triangle \operatorname{SR}(\mathbf{W}_K) = -\eta {P} \mathcal{O}(\mu_1^2\phi^2) \mu_1^2 {R},
\end{equation}
where \( {R} = \left[ \left( 1 - \operatorname{SR}(\mathbf{W}_K) \right) \phi^2 + \left( \operatorname{SR}(\mathbf{Z}) - 1 \right) \right] \). This reduction drives \(\operatorname{SR}(\mathbf{W}_Q)\) and \(\operatorname{SR}(\mathbf{W}_K)\) toward \(\operatorname{SR}(\mathbf{Z})\), intensifying parametric singularities.
\end{theorem}

A detailed proof is provided in Appendix~\ref{appendix:parametric_singularity_trend}. Given the dominance of \(\mu_1\) in Equation~\ref{eq:qkgrad-simple}, gradient updates emphasize the dominant singular direction in the representation space. Strong alignment between the top singular vectors of the parametric and representation spaces causes these updates to disproportionately increase the largest singular values of \(\mathbf{W}\), reducing \(\operatorname{SR}(\mathbf{W}_Q)\) and \(\operatorname{SR}(\mathbf{W}_K)\), an effect amplified by a large learning rate. This leads to stronger alignment between the two spaces, as demonstrated in the subsequent theorem.

\begin{theorem}[Amplification of Singularity Alignment]
Let \( \mathbf{W}_{QK} \) be updated using the gradient \( \nabla_{\mathbf{W}_{QK}} \mathcal{J}\) as in Equation \ref{eq:qkgrad-simple}. The singular alignment \( \phi \) between \( \mathbf{W}_{QK} \) and \(\mathbf{Z} \) increases after the update. The change in alignment is given by:
\begin{equation}
\triangle \phi = -\eta P \mathcal{O}(\sigma_1^2 \mu_1^2\phi^2) \mu_1^2 \phi \sum_{k=2}^d\frac{(\mathbf{v}_k^\top \boldsymbol{\beta}_1)^2}{\sigma_1^2 - \sigma_k^2},
\end{equation}
indicating a positive update of \( \phi \).
\end{theorem}

A detailed proof is shown in Appendix \ref{appendix:alignment}. The intuition behind the increase in alignment \(\phi\) follows Theorem~\ref{theorem:parametric}. Due to the prominence of \(\mu_1\) in the representation space, gradient updates primarily perturb the dominant direction \(\mathbf{v}_1\) of the weight matrices toward \(\boldsymbol{\beta}_1\), strengthening the alignment between \(\mathbf{v}_1\) and \(\boldsymbol{\beta}_1\), and thus enhancing singularity alignment.


\begin{theorem}[Amplification of Representation Singularity]
We study the Key-Query inner product in a one-layer Transformer: \(\mathbf{Z}_K^\top \mathbf{Z}_Q= \mathbf{Z} \mathbf{W}_{QK} \mathbf{Z}^\top \), then the stable rank of \( \mathbf{Z}_K \) and \( \mathbf{Z}_Q \) is:
\begin{equation}
    \operatorname{SR}(\mathbf{Z}_K) = \operatorname{SR}(\mathbf{Z}_Q) \approx 1+\frac{\left[ \operatorname{SR}(\mathbf{W}_K)-1\right]\left[ \operatorname{SR}(\mathbf{Z})-1 \right]}{(d-1)\phi^2}
\end{equation}
Let \( \mathbf{W}_{QK} \) be updated using the simplified gradient \( \nabla_{W_{QK}} \mathcal{J}\) as in Equation \ref{eq:qkgrad-simple}. Then if the following condition being satisfied: $\operatorname{SR}(\mathbf{W}_K) >  1 + \frac{\phi^2\left[ \operatorname{SR}(\mathbf{Z})-1 \right]}{{(d-1)}^2},$ $\operatorname{SR}(\mathbf{Z}_K)$ decreases.
\end{theorem}

A detailed proof is shown in Appendix \ref{appendix:feature_singularity}.
Here, \(\mathbf{Z}\) represents the fixed input representation space in the one-layer Transformer setting. The output representation singularities are governed by \(\operatorname{SR}(\mathbf{W}_K)\) and singularity alignment \(\phi\), where a decreasing \(\operatorname{SR}(\mathbf{W}_K)\) or increasing \(\phi\) reduces \(\operatorname{SR}(\mathbf{Z}_K)\) and \(\operatorname{SR}(\mathbf{Z}_Q)\), triggering rank co-collapse in parametric and representation spaces. In multi-layer models, shallow layer outputs feed deeper layers, propagating singularity reinforcement across layers and amplifying the \textit{curse of singularities} spatially and temporally.



\subsection{Singularity and Training Instability}
The degree of singularity strongly correlates with the Frobenius norm of gradients, a key indicator of training stability. As shown in Figure~\ref{figure:metrices}(d), gradient norms rise with increasing singularity during training. Figure~\ref{figure:metrices}(top) illustrates a training failure, further highlighting this connection. Around step 85,800, the loss surges from 4 to 7, accompanied by a sharp gradient norm spike and a rapid increase in both parametric and representation singularities.

Building on prior foundations, we establish that the gradient Frobenius norm is both lower and upper bounded by parametric singularity. Since the proof of Theorem~\ref{theorem:parametric} shows that training simultaneously increases \(\sigma_1\) while reducing \(\operatorname{SR}(\mathbf{W}_K)\), we use \(\sigma_1\) as a proxy for the growing parametric singularity.



\begin{theorem}[Bounds on Gradient Frobenius Norm] 
The Frobenius norm of the \(QK\)-gradient \(\nabla_{\mathbf{W}_{QK}} \mathcal{J}\) is bounded as follows:

1. Lower Bound: Assuming a fixed Frobenius norm \(\|\mathbf{W}_{QK}\|_{\text{F}}^2 = \sum_{k=1}^d \sigma_k^2 = M > 0\) at a given step, the gradient norm satisfies:
\begin{equation}
\label{eq:lower-bound}
    \|\nabla_{\mathbf{W}_{QK}} \mathcal{J}\|_{\text{F}} \geq K \alpha_1 \sigma_1^2(\mathbf{W}_{QK}),
\end{equation}
where \( K = P \cdot \| {\sum_{t=1}^d\mu_t^2\boldsymbol{\beta}_t\boldsymbol{\beta}_t^\top} \| > 0 \) and  \( \alpha_1 =\left( \sum_{r=1}^d \mu_r \cdot (\mathbf{v}_1^\top \boldsymbol{\beta}_r) \right)^2. \) 

2. Upper Bound:
\begin{equation}
\label{eq:upper-bound}
     \|\nabla_{\mathbf{W}_{QK}} \mathcal{J}\|_{\text{F}} \leq C \sigma_{1}^2(\mathbf{W}_V){\left[ \sigma_1(\mathbf{W}_{F_1})\sigma_1(\mathbf{W}_{F_2})+1\right]}^2\sigma_1(\mathbf{W}_{QK}),
\end{equation}
where $C = \frac{1}{d}\sum_{i,j,a,b \in [T]}\mathbb{E}\left[ \tilde{\gamma}_a^i\tilde{\gamma}_b^j\right]\mathcal{O}(\mu_1^4\phi^2) \operatorname{SR}(\mathbf{Z}).$ \
\end{theorem}

A detailed proof is shown in Appendix \ref{appendix:gradient_upperbound}. As shown in Equation~\ref{eq:lower-bound}, the gradient norm’s lower bound increases with parametric singularity, driven by \(\sigma_1^2\). This trend, observed in both stable and unstable training regimes (Figure~\ref{figure:metrices}(d)), suggests that reducing the learning rate in later training stages is essential to accommodate larger gradients. Similarly, Equation~\ref{eq:upper-bound} reveals that the upper bound grows rapidly with increasing parametric singularity, amplifying the risk of unstable gradient updates. Together, these dynamics—where gradient norm bounds relax with growing singularity—establish a mechanistic connection between parametric singularities and training instability.




%% file: sections/3_method.tex
\section{Method}
\label{method}
To stabilize DNN training, we propose Parametric Singularity Smoothing (PSS), a lightweight module that integrates seamlessly into DNNs. Using gradient norms as instability indicators, PSS smooths weight matrix singular spectra, reducing parametric singularities and tightening gradient norm bounds. Crucially, PSS preserves network performance and can restore trainability post-instability.

\subsection{Design of PSS}
\textbf{Detection}. \hspace{0.06em} 
Prior analyses has established a link between gradient norm spikes, heightened singularity, and training instability. We leverage gradient norms as early instability indicators, computed during backpropagation. We define the ratio \(\mu_{t+1} = \frac{\|\mathbf{g}^{t+1}\|_{\text{F}}}{\|\mathbf{g}_{avg}^t\|_{\text{F}}}\), where \(\mathbf{g}^{t+1}\) is the gradient at step \(t+1\), and \(\mathbf{g}_{avg}^t = (1 - \alpha) \mathbf{g}_{avg}^{t-1} + \alpha \mathbf{g}^t\) is the exponentially weighted moving average of gradients, with smoothing coefficient \(\alpha\). If \(\mu_{t+1} \geq \tau\), where \(\tau\) is a threshold, potential instability is flagged, triggering a smoothing operation to curb parametric singularity growth. The choice of \(\tau\) is robust: a lower \(\tau\) may cause false positives without disrupting training, while a higher \(\tau\) risks missing instabilities, yet PSS can still recover stability post-instability.

\textbf{Protection}. \hspace{0.06em} 
Upon detecting instability, we smooth the singular spectrum of the parameter matrix \(\mathbf{W}\) to curb gradient norm growth. For efficiency, we compute dominant singular values and vectors up to the matrix’s stable rank using the Power Iteration Method~\citep{golub2013matrix} for \(\sigma_1\) and Dominant Direction Decomposition (DDD)~\citep{halko2009finding} to obtain \(\mathbf{W}_{dominant} = \sum_{i=1}^{\lfloor \operatorname{SR}(\mathbf{W}) \rfloor} \sigma_i \mathbf{u}_i \mathbf{v}_i^\top\). We apply a smoothing function \(f_{\text{smooth}}\) to yield \(\mathbf{\sigma}^* = f_{\text{smooth}}(\mathbf{\sigma})\), reparameterizing as \(\mathbf{W}_{dominant}^* = \sum_{i=1}^{\lfloor \operatorname{SR}(\mathbf{W}) \rfloor} \sigma_i^* \mathbf{u}_i \mathbf{v}_i^\top\), resulting in \(\mathbf{W}^* = \mathbf{\mathbf{W}} - \mathbf{W}_{dominant} + \mathbf{W}_{dominant}^*\). The flexible \(f_{\text{smooth}}\) supports methods like Logarithmic Scaling, Softplus, Softmax, or clipping, preserving singular value order and increasing SR, with the choice tailored to the task.

\subsection{Benefits of PSS}
PSS is a lightweight module that seamlessly integrates into DNNs to detect and prevent instability without degrading performance. It can restore trainability post-instability while preserving all singular vectors and their learned features. By smoothing the singular spectrum, PSS controls gradient norms and reduces excessive dominance of certain directions, allowing smaller ones to emerge. This enables progressive refinement of the singular value distribution during training.

%% file: sections/4_experi.tex
\section{Experiments}

In this section, we evaluate the performance of Parametric Singularity Smoothing (PSS) through a series of experiments across various scenarios. Our primary goal is to demonstrate how PSS effectively preserves stability during training without degrading model performance. We also highlight its efficiency and lightweight nature, ensuring minimal computational overhead.

\subsection{Experimental Setup}
\label{experiments:experimental setup}
\textbf{Configuration.} \hspace{0.06em}  
We evaluate PSS across model scales (110M to 7B parameters) using BERT~\citep{devlin2019BERT}, GPT-2~\citep{radford2019language}, and Llama-3~\cite{llama3}, with primary results on BERT-base and GPT-2-Medium (larger models in Appendix~\ref{appendix:Experiment Results on Larger Models}). Pretraining uses Wikitext~\citep{merity2016wikitext} (BERT), Amazon-review/OpenWebText~\citep{Gokaslan2019OpenWeb} (GPT-2), and English Wikipedia (Llama-3). Tests cover instability scenarios (learning rate, warmup, batch size) with comparisons to Gradient Clipping(GC), Orthogonal Regularization(OR)~\citep{brock2017neural, brock2018large}, and Query-Key Normalization(QK-Norm)~\cite{henry2020querykeynormalizationtransformers}. See Appendix~\ref{appendix:model-configuration} for settings.

\textbf{Evaluation.} \hspace{0.06em} 
We assess PSS's effectiveness by: (1) measuring loss explosion frequency across identical configurations~\cite{zhai2023reparam}, (2) evaluating expanded stable learning rate ranges and corresponding performance, and (3) comparing computational overhead~\citep{rasley2020deepspeed} to highlight PSS's efficiency. Results demonstrate PSS's robustness in rescue time and smoothing policy choice, while downstream task performance confirms no model degradation.

\subsection{Effectiveness in Preventing Instability}

\begin{table*}
\centering
\caption{Stability test results for BERT-base and GPT-2-Medium. Entries follow the format x/y (z), where y indicates the total number of evaluations conducted, x reflects the count of instability instances observed during training, and z reports the final test loss.}
\begin{tabular}{@{}ccccccc@{}}
\toprule
\multirow{2}{*}{\diagbox{Method}{Model}} & \multicolumn{3}{c}{BERT-base} & \multicolumn{3}{c}{GPT-2-Medium} \\ \cmidrule(l){2-4} \cmidrule(l){5-7} 
& lr=1e-4 & lr=2e-4 & lr=4e-4 & lr=5e-4 & lr=1e-3 & lr=4e-3 \\ \midrule
Naive & 
$0/3_{3.0\pm0.2}$ & $3/3_{7.5\pm0.1}$ & $3/3_{7.5\pm0.1}$ &
$0/3_{3.8\pm0.1}$ & $3/3_{7.4\pm0.1}$ & $3/3_{7.5\pm0.1}$ \\
GC &
$0/3_{2.9\pm0.1}$ & $0/3_{2.7\pm0.1}$ & $3/3_{7.4\pm0.1}$ &
$0/3_{3.7\pm0.1}$ & $0/3_{3.9\pm0.1}$ & $3/3_{7.5\pm0.1}$ \\
OR &
$0/3_{3.1\pm0.1}$ & $1/3_{2.9\pm0.1}$ & $3/3_{7.4\pm0.1}$ &
$0/3_{3.8\pm0.1}$ & $3/3_{7.5\pm0.1}$ & $3/3_{7.4\pm0.1}$ \\ 
QK-Norm &
$0/3_{3.0\pm0.1}$ & $0/3_{2.7\pm0.1}$ & $0/3_{2.8\pm0.1}$ &
$0/3_{3.8\pm0.1}$ & $0/3_{3.9\pm0.1}$ & $3/3_{7.2\pm0.4}$ \\ 
\midrule
PSS &
$0/3_{3.0\pm0.1}$ & $0/3_{2.8\pm0.1}$ & $0/3_{2.6\pm0.1}$ &
$0/3_{3.8\pm0.1}$ & $0/3_{4.0\pm0.1}$ & $0/3_{4.5\pm0.1}$ \\ \bottomrule
\end{tabular}
\label{table:bert-stability}
\end{table*}

\textbf{Stability Maintenance.}
\hspace{0.06em} 
To evaluate PSS’s stability, we compare it against three stabilization methods across various hyper-parameter settings, primarily learning rates. Each PSS-baseline comparison uses the same random seed to assess instability prevention, repeated with three distinct seeds for consistency. Additional tests on warmup steps and batch size are in Appendix~\ref{appendix:Experiment Results on different Hyper-parameters}.

Table~\ref{table:bert-stability} shows GC and OR stabilize training only at low LRs, failing at higher LRs, while QK-Norm, though improved, falters at larger LRs. PSS consistently ensures stability across all settings and seeds. Appendix~\ref{appendix:loss-PSS&others} presents BERT-base test loss curves at LRs of 1e-4 and 4e-4, demonstrating PSS’s ability to mitigate loss explosions with minimal spikes and no extra recovery steps.

\begin{figure}[h]
    \vspace{-0.5\baselineskip}  
    \centering
    \includegraphics[width=\linewidth]{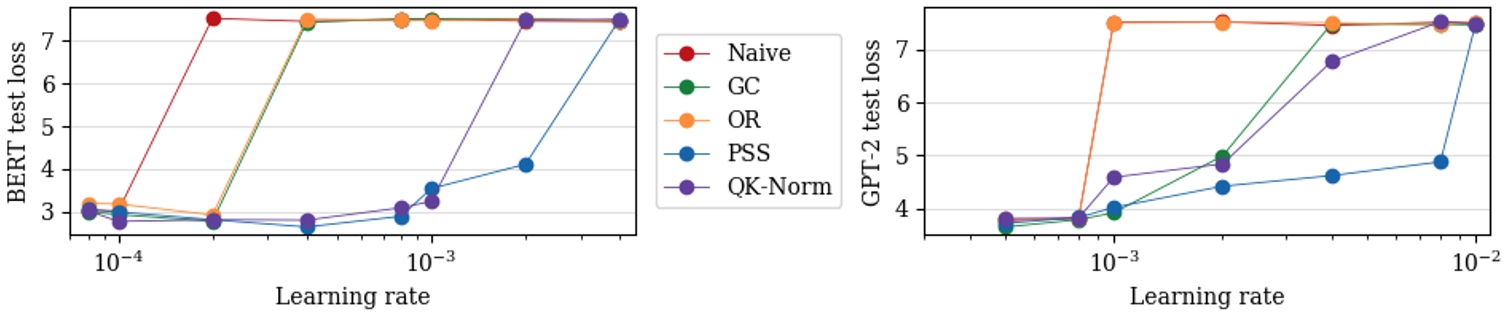} 
     
    \caption{Final loss across LRs for each stabilization method (BERT-base left, GPT-2-Medium right).}
    \label{figure:loss-steps-at-convergence}
    \vspace{-0.5\baselineskip}  
\end{figure}
\textbf{Model Performance.}
\hspace{0.06em} 
Fig.~\ref{figure:loss-steps-at-convergence} compares final test loss across various LRs for BERT-base and GPT-2-Medium. PSS significantly expands the stable LR range, achieving up to a 10-fold increase compared to GC and OR. QK-Norm underperforms PSS at large LRs, particularly on GPT-2, where its loss surges with increasing LRs. For larger models (BERT-large, GPT-2-Large, GPT-2-XL), PSS extends the stable LR range by up to 5 times without sacrificing convergence, enhancing hyperparameter tuning flexibility.

\begin{wraptable}{l}{9cm}
\vspace{-\baselineskip}  
\centering
\caption{Computational cost comparison of different methods in terms of floating-point (Fp) computations and execution time for both a single step and 100k steps.}
\label{tab:computation-cost}
\begin{tabular}{@{}ccccc@{}}
\toprule
\multirow{4}{*}{Method} &
\multicolumn{2}{c}{CPU+GPU cost} & \multicolumn{2}{c}{CPU+GPU cost} \\ 
& \multicolumn{2}{c}{@ 1 step} & \multicolumn{2}{c}{@ 100k step}\\ \cmidrule(l){2-3} \cmidrule(l){4-5}
& Fp cost & Time & Fp cost & Time \\
& {\scriptsize (TFLOPs)} & {\scriptsize (ms)} & {\scriptsize (PFLOPs)} & {\scriptsize (hours)} \\\midrule
Naive   & 9.33  & $781\pm10$  & 933  & $9.15\pm0.08$  \\
GC      & 9.33  & $784\pm12$  & 936  & $9.14\pm0.10$  \\
OR      & 11.46 & $864\pm15$  & 1146 & $10.04\pm0.10$  \\
QK-Norm & 9.33 & $780\pm23$  & 1146 & $9.20\pm0.13$  \\ \midrule
PSS     & 9.36  & $1979\pm25$ & 936  & $9.17\pm0.09$  \\ \bottomrule
\end{tabular}
\vspace{-\baselineskip}  
\end{wraptable}

\subsection{Computational Efficiency of PSS}
\textbf{Per-Step Computational Cost.}
\hspace{0.06em} 
PSS detects instability via gradient norms with a minimal cost of \( O(mn) \) for a parameter matrix \( \mathbf{W} \in \mathbb{R}^{m \times n} \). Its smoothing, using Dominant Direction Decomposition (DDD), has a complexity of \( \mathcal{O}(mn\log k) \), where \( k = \lfloor \operatorname{SR}(\mathbf{W}) \rfloor \), comparable to a forward-backward pass (\( \mathcal{O}(mnb) \)) for a feature matrix \( \mathbf{Z} \in \mathbb{R}^{n \times b} \), as \( \log k \) and \( b \) are similar in scale. For BERT-base (\( m=768 \), \( n=3072 \), \( b=128 \)), \( k<50 \); for LLaMA3-70B (\( m=8192 \), \( n=28672 \), \( b=8192 \)), \( k<8192 \) (exact \( k \) pending). Since \( m, n \gg b \) and \( \log k < k < \max(m,n) \), \( \log k \approx b \), ensuring low overhead. Table~\ref{tab:computation-cost} (``1 step'' section) shows PSS increases CPU + GPU time by up to 2.40× a baseline step.

\textbf{Total Training Overhead.}
\hspace{0.06em} 
PSS incurs minimal overall overhead, activating on average in 0.1\% of training steps, with a maximum below 0.25\% (see Appendix~\ref{appendix:PSS activation freq}). This low frequency holds across model sizes and architectures, increasing slightly with larger learning rates due to more gradient spikes, showcasing PSS’s adaptability. Table~\ref{tab:computation-cost} (``100k step'' section) shows that training BERT-base with a large LR adds only 0.21\% extra training time over the baseline, highlighting PSS’s efficiency.

\subsection{Robustness of PSS}

\begin{figure*}[h]
    \centering
    \vspace{-1\baselineskip} 
    \subfigure[]{
        \includegraphics[width=.45\linewidth]{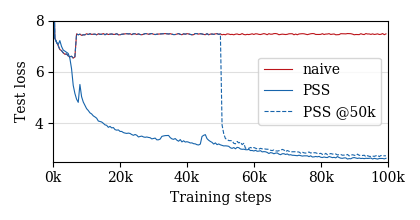}
    }
    \subfigure[]{
        \includegraphics[width=.45\linewidth]{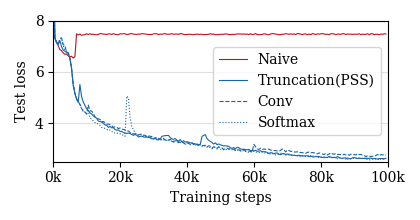}
    }
    \vspace{-1\baselineskip} 
    \caption{For BERT-base trained at a learning rate of 4e-4: (a) Test loss curves comparing the unstable naive baseline with PSS applied at various stages. (b) Test loss curves using different smoothing policies upon instability detection.}
    \label{figure:robustness}
\end{figure*}
PSS demonstrates robust performance across multiple critical dimensions:

\textbf{Rescue Robustness.} \hspace{0.06em} 
Fig.~\ref{figure:robustness}(a) shows PSS's robustness across instability stages. Proactively applied, PSS induces only minor, transient loss spikes, averting future instability. Post-divergence, it quickly mitigates gradient vanishing and restores training. Unlike prior methods—often ineffective or reliant on slow trial-and-error or checkpoint recovery—PSS detects instability reliably: false positives don’t disrupt training, and it stabilizes even after full divergence, offering an efficient solution.

\textbf{Smoothing Policy Robustness.} \hspace{0.06em} 
As shown in Fig.\ref{figure:robustness}, PSS is effective across a variety of smoothing policies, all of which yield stable training outcomes. Further details on each smoothing policy are provided in Appendix\ref{appendix:smoothing-policy}. This suggests that once parametric singularity is addressed, the specific choice of smoothing policy is less critical, emphasizing the method's robustness in policy selection. 

\textbf{Compatibility Robustness.} \hspace{0.06em} 
PSS is highly adaptable, requiring no specific model architecture, optimizer, or dataset. This ensures easy integration with minimal code changes. Moreover, PSS is orthogonal to existing stabilization techniques, complementing and enhancing them without conflict, further improving training stability.

\subsection{Performance on Downstream Tasks}
\textbf{Comparable Accuracy.}
\hspace{0.06em} 
To evaluate PSS’s impact on model performance, we conduct experiments on four downstream tasks using BERT-base trained for 100k steps, as shown in Table~\ref{table:downstream-task}. PSS not only ensures training stability but also achieves comparable or superior accuracy compared to the naive baseline, particularly at larger learning rates (LRs).

\textbf{Benefits of Large LRs.}
\hspace{0.06em} 
Table~\ref{table:downstream-task} demonstrates that PSS with a large LR of 4e-4 achieves better loss and downstream accuracy in 70k steps compared to 100k steps at LR 1e-4. This indicates PSS reduces the training steps required at large LRs, lowering computational costs while unlocking the potential of higher LRs for improved performance.
\begin{table}[h]
\caption{Downstream tasks on BERT-base.}
\centering
\begin{tabular}{@{}cccccccc@{}}
\toprule
\multirow{2}{*}{Method} 
& \multirow{2}{*}{test loss} & CoLA & SST-2 & MRPC & MNLI & QNLI & RTE \\ 
& & (acc) & (acc) & (acc) & (acc) & (acc) & (acc) \\ 
\midrule

Naive @ 100k steps(lr=1e-4) & 3.05 & 50.8\% & 85.2\% & 82.4\% & 81.6\% & 84.3\% & 81.9\% \\ \midrule
PSS @ 100k steps(lr=1e-4) & 2.95 & 51.9\% & 84.9\% & 83.3\% & 82.0\% & 85.0\% & 81.2\% \\
PSS @ 100k steps(lr=2e-4) & 2.78 & 51.6\% & 86.0\% & 83.4\% & 82.8\% & 85.8\% & 83.3\% \\ 
PSS @ 100k steps(lr=4e-4) & 2.64 & 52.2\% & 85.6\% & 84.8\% & 82.5\% & 85.8\% & 84.5\% \\ \midrule
PSS @ 70k steps(lr=2e-4) & 2.92 & 51.0\% & 85.8\% & 81.7\% & 82.5\% & 83.7\% & 82.4\% \\ 
PSS @ 70k steps(lr=4e-4) & 2.84 & 52.3\% & 86.1\% & 84.6\% & 81.9\% & 85.4\% & 82.8\% \\\bottomrule
\end{tabular}

\label{table:downstream-task}
\end{table}

%% file: sections/5_related_work.tex
\section{Related Work}
Training instability has been widely studied from a theoretical perspective. Techniques like meticulous weight initialization~\citep{glorot2010understanding, he2015delving, mishkin2015all, hu2020provable} and normalization~\citep{ba2016layer, ioffe2015batch, ulyanov2016instance} are critical for stabilizing neural networks. Additionally, research has examined the effects of activation non-linearity~\citep{hu2016revisiting}, skip connections~\citep{szegedy2017inception}, and model depth and width~\citep{hanin2018neural, yang2018deep} on stability. In practice, gradient clipping~\citep{allen2019convergence} is a common method for managing instability by capping gradient magnitudes. Adaptive methods, such as those proposed by~\cite{Clippy2023}, improve on this by incorporating optimizers like Adagrad. Other approaches include selectively freezing parameters in ViT models~\citep{chenempirical} and using spectral normalization to prevent attention entropy collapse~\citep{zhai2023reparam}, further stabilizing training. Despite the aforementioned stability methods, models still encounter instability under large LRs, and these methods prove ineffective once instability arises. Our PSS method significantly broadens the usable LR range and ensures reliable training.

%% file: sections/6_conclusion.tex
\section{Conclusion}
We expose the phenomenon of rank co-collapses in the parametric and representation spaces as \textit{the curse of singularities} and deem it an essential cause for training instabilities. 
We thoroughly analyze the theoretical properties behind this curse and acquire several key findings.
In light of those discoveries, we propose a method, Parametric Singularity Smoothing (PSS), to smooth the singular spectrum of the weight matrices when necessary, which could effectively prevent unstable training without compromising performance.
More favourably, PSS could restore trainability even after perceivable instability.
As a lightweight, generalized, flexible, and compatible module, PSS can be integrated into various networks to save time and computation on undesired re-training, especially for large language models.
Moreover, we open a new gate for analyzing the instability behaviour in the parametric space.

%% file: sections/appendix.tex
\section{Theoretical Proof.}
\label{appendix:theoretical proof}

\subsection{Parametric Singularity Trend.}
\label{appendix:parametric_singularity_trend}

\begin{theorem}[Amplification of Parametric Singularity]
The gradient of the loss \(\mathcal{J}\) with respect to \(\mathbf{W}_{QK}\) amplifies parametric singularities, approximated as:
\begin{equation}
    \nabla_{\mathbf{W}_{QK}} \mathcal{J} = {P}  \mathcal{O}(\sigma_1^2\mu_1^2\phi^2) \sum_{t=1} \mu_t^2 \boldsymbol{\beta}_t \boldsymbol{\beta}_t^\top,
\end{equation}
where \( P = \frac{1}{d} \sum_{i,j,a,b \in [T]} \mathbb{E} \left[ \tilde{\gamma}^i_a \tilde{\gamma}^j_b {P}_{ij} \right] \) is negative. For large \( T \) and \(\operatorname{SR}(\mathbf{W}_K) > 1 + \frac{\operatorname{SR}(\mathbf{\mathbf{Z}}) - 1}{\phi^2}\), a gradient update with learning rate \(\eta\) reduces the stable rank of \(\mathbf{W}_Q\) and \(\mathbf{W}_K\) by:
\begin{equation}
\label{eq:srw-decrease}
    \triangle \operatorname{SR}(\mathbf{W}_Q) = \triangle \operatorname{SR}(\mathbf{W}_K) = -\eta {P} \mathcal{O}(\mu_1^2\phi^2) \mu_1^2 {R},
\end{equation}
where \( {R} = \left[ \left( 1 - \operatorname{SR}(\mathbf{W}_K) \right) \phi^2 + \left( \operatorname{SR}(\mathbf{Z}) - 1 \right) \right] \). This reduction drives \(\operatorname{SR}(\mathbf{W}_Q)\) and \(\operatorname{SR}(\mathbf{W}_K)\) toward \(\operatorname{SR}(\mathbf{Z})\), intensifying parametric singularities.
\end{theorem}

\begin{proof}
    Through semantic basis composition, the token representations \( \mathbf{x}_{a} \), \( \mathbf{x}_{b} \), and \( \mathbf{x}_{T} \) in Equation~\ref{equation:origin_gradient_form} can be expressed as linear combinations of the semantic basis vectors:
    \[
\mathbf{x}_{a}=\sum_{i=1}^dc_{a,i}\boldsymbol{\beta}_i,\mathbf{x}_{b}=\sum_{i=1}^dc_{b,i}\boldsymbol{\beta}_i,\mathbf{x}_{T}=\sum_{i=1}^dc_{T,i}\boldsymbol{\beta}_i
    \]
    Substituting this into Equation~\ref{equation:origin_gradient_form} gives the semantic feature decomposition form of the $QK$-gradient $\nabla_{\mathbf{W}_{QK}} \mathcal{J}$.
    \begin{align*}
        \nabla_{\mathbf{W}_{QK}} \mathcal{J} &= \frac{1}{d} \sum_{i,j,a,b \in [T]} \mathbb{E} \left[ \tilde{\gamma}^i_a \tilde{\gamma}^j_b {P}_{ij}(\mathbf{x}_b^\top \mathbf{W}_{QK} \mathbf{x}_T) \mathbf{x}_a \mathbf{x}_T^\top \right] \\
        &= \frac{1}{d} \sum_{i,j,a,b \in [T]} \mathbb{E} [ \tilde{\gamma}^i_a \tilde{\gamma}^j_b {P}_{ij} {(\sum_{r=1}^dc_{b,r}\boldsymbol{\beta}_r^\top)}\mathbf{W}_{QK}(\sum_{s=1}^dc_{T,s}\boldsymbol{\beta}_s) (\sum_{t=1}^dc_{a,t} \boldsymbol{\beta}_t){(\sum_{t^\prime=1}^dc_{T,t\prime}\boldsymbol{\beta}_{t^\prime}^\top)} ] \\
        &= \frac{1}{d} \sum_{i,j,a,b \in [T]} \mathbb{E} [ \tilde{\gamma}^i_a \tilde{\gamma}^j_b {P}_{ij} \sum_{r,s}c_{b,r}c_{T,s}\left(\boldsymbol{\beta}_r^\top \mathbf{W}_{QK}\boldsymbol{\beta}_s \right) \sum_{t=1}^dc_{a,t}c_{T,t}\boldsymbol{\beta}_t\boldsymbol{\beta}_t^\top]
    \end{align*}
Due to Semantic Basis Composition assumption, the above expression can be derived:
\begin{align*}
        \nabla_{\mathbf{W}_{QK}} \mathcal{J} &= \frac{1}{d} \sum_{i,j,a,b \in [T]} \mathbb{E} [ \tilde{\gamma}^i_a \tilde{\gamma}^j_b {P}_{ij} \sum_{r,s}c_{b,r}c_{T,s}\left(\boldsymbol{\beta}_r^\top \mathbf{W}_{QK}\boldsymbol{\beta}_s \right) \sum_{t=1}^dc_{a,t}c_{T,t}\boldsymbol{\beta}_t\boldsymbol{\beta}_t^\top] \\
        &= \frac{1}{d} \sum_{i,j,a,b \in [T]} \mathbb{E} \left[ \tilde{\gamma}^i_a \tilde{\gamma}^j_b {P}_{ij} \sum_{r,s}\mu_r\mu_s\left(\boldsymbol{\beta}_r^\top \mathbf{W}_{QK}\boldsymbol{\beta}_s \right) \sum_{t=1}^d\mu_t^2\boldsymbol{\beta}_t\boldsymbol{\beta}_t^\top\right] \\
        &= {P} \left[\sum_{r,s}\mu_r\mu_s\left(\boldsymbol{\beta}_r^\top \mathbf{W}_{QK}\boldsymbol{\beta}_s \right)\right]\sum_{t=1}^d\mu_t^2\boldsymbol{\beta}_t\boldsymbol{\beta}_t^\top,
    \end{align*}
where, ${P}=\frac{1}{d}\sum_{i,j,a,b \in [T]}\mathbb{E}\left[ \tilde{\gamma}^i_a \tilde{\gamma}^j_b{P}_{ij}\right].$

Since \( \mathbf{W}_{QK} := \mathbf{W}_Q^\top \mathbf{W}_K \) and the query-key parameters are optimized jointly, we can perform an eigenvalue decomposition on the matrix \( \mathbf{W}_{QK} \), such that \( \mathbf{W}_{QK} = \sum_k \sigma_k^2 \mathbf{v}_k \mathbf{v}_k^\top \). By substituting into the expression for \( \nabla_{\mathbf{W}_{QK}} \mathcal{J} \), we obtain:
\begin{equation}
\label{equation:gradient_approx}
    \nabla_{\mathbf{W}_{QK}} \mathcal{J} = {P} \left[\sum_{k=1}^d\sigma_k^2\sum_{r,s=1}^d\mu_r\mu_s\left( \boldsymbol{\beta}_r^\top \mathbf{v}_k\right)\left( \mathbf{v}_k^\top \boldsymbol{\beta}_s\right)\sum_{t=1}^d\mu_t^2\boldsymbol{\beta}_t\boldsymbol{\beta}_t^\top\right].
\end{equation}
$\sum_{k=1}^d \sigma_k^2 \sum_{r,s=1}^d \mu_r \mu_s \left( \boldsymbol{\beta}_r^\top \mathbf{v}_k \right) \left( \mathbf{v}_k^\top \boldsymbol{\beta}_s \right)$ is a scalar. 

we decompose:

$\sum_{k=1}^d \sigma_k^2 \sum_{r,s=1}^d \mu_r \mu_s (\beta_r^\top \mathbf{v}_k)(\mathbf{v}_k^\top \beta_s)
= \sigma_1^2 \mu_1^2 (\mathbf{v}_1^\top \beta_1)^2 + 2 \sigma_1^2 \sum_{r=2}^d \mu_1 \mu_r (\beta_1^\top \mathbf{v}_1)(\mathbf{v}_1^\top \beta_r) $ $+ \sigma_1^2 \sum_{r,s=2}^d \mu_r \mu_s (\beta_r^\top \mathbf{v}_1)(\mathbf{v}_1^\top \beta_s) + \sum_{k=2}^d \sigma_k^2 \sum_{r,s=1}^d \mu_r \mu_s (\beta_r^\top \mathbf{v}_k)(\mathbf{v}_k^\top \beta_s).$

Due to the representation space naturally exhibits a low stable rank and the Strong Singularity Alignment Assumption, this is bounded as:
$\leq \sigma_1^2 \mu_1^2 \left[ (1 - \epsilon) + 2 \frac{\epsilon}{d-1} \frac{1}{\operatorname{SR}(\mathbf{Z})} + \frac{\epsilon}{d-1} \frac{1}{\operatorname{SR}(\mathbf{Z})^2} + \frac{1}{d} \frac{1}{\operatorname{SR}(\mathbf{Z})^2} (SR(\mathbf{W}) - 1) \right].$

The first term, $ (1 - \epsilon) $, dominates as $\mu_1$-dominance and for large $ d $ (e.g., 768 in BERT-base), which suppresses subsequent terms. The exact value of $ \epsilon $ minimally affects the magnitude of subsequent terms, preserving the dominance of $ 1 - \epsilon $. Thus, we approximate the expression as $ \mathcal{O}(\sigma_1^2 \mu_1^2 \phi^2).$

So we have:
\[
\sum_{k=1}^d \sigma_k^2 \sum_{r,s=1}^d \mu_r \mu_s \left( \boldsymbol{\beta}_r^\top \mathbf{v}_k \right) \left( \mathbf{v}_k^\top \boldsymbol{\beta}_s \right) = \mathcal{O}(\sigma_1^2\mu_1^2 {\left( \mathbf{v}_1^\top\boldsymbol{\beta}_1\right)}^2)
\]
So the gradient expression can be further simplified as follows:
\begin{equation}
   \nabla_{\mathbf{W}_{QK}} \mathcal{J} \approx {P}  \mathcal{O}(\sigma_1^2\mu_1^2\phi^2) \sum_{t=1} \mu_t^2 \boldsymbol{\beta}_t \boldsymbol{\beta}_t^\top,
\end{equation}

When \(\mathbf{W}_{QK} \) is updated via gradient descent as \( \mathbf{W}_{QK}^\prime \leftarrow \mathbf{W}_{QK} - \eta \nabla_{\mathbf{W}_{QK}} \mathcal{J}  \), according to matrix perturbation theory, the change in the eigenvalues $\triangle\lambda_l = \triangle\sigma_l^2$is given by:
\[
 \triangle\lambda_l = \triangle\sigma_l^2 = \mathbf{v}_l^\top\left( -\eta \nabla_{\mathbf{W}_{QK}} \mathcal{J} \right)\mathbf{v}_l.
\]
By substituting the form of $\nabla_{\mathbf{W}_{QK}} \mathcal{J} $, the change in the square of each singular value is obtained as:
\begin{equation}
\label{equation:delta_sigma}
    \triangle\sigma_l^2 = -\eta{P} \mathcal{O}(\sigma_1^2\mu_1^2\phi^2)\sum_{t=1}\mu_t^2 \left(\mathbf{v}_l^\top\boldsymbol{\beta}_t\right) 
\end{equation}

Therefore, the change in the square of the largest singular value of the parameter matrix is given by:

\begin{align*}
    \triangle\sigma_1^2 &= -\eta {P} \mathcal{O}(\sigma_1^2\mu_1^2\phi^2)\sum_{t=1}\mu_t^2 \left(\mathbf{v}_1^\top\boldsymbol{\beta}_t\right) \\
    &= -\eta {P} \mathcal{O}(\sigma_1^2\mu_1^4\phi^4).
\end{align*}

Recall the definition of $P$:  
\[
P=\frac{1}{d}\sum_{i,j\in[T]}\mathbb{E}[\sum_{a,b\in [T]}\tilde{\gamma}_a^i\tilde{\gamma}_b^j ] P_{ij}.
\]  
We derive $P$’s sign by analyzing the expectation, splitting into two cases based on the independence of $\tilde{\gamma}^i$ and $\tilde{\gamma}^j$: (1) $i = j$, and (2) $i \neq j$.

\paragraph{Case 1$(i=j)$:}

(Coordinate Dependence)
When $i=j$, $\tilde{\gamma}^i = \tilde{\gamma}^j$, introducing dependence among coordinates in the truncated $\tilde{\gamma}^i$. If no truncation occurs at index $i$ ($\tilde{\gamma}_i^i \neq 0$), it implies no truncation at any $a\neq i$ ($\tilde{\gamma}_a^i\neq 0$). Define the non-truncation events:  
\[
E_i = \{|\frac{T-1}{T^2}\omega_i+\frac{1}{T^2}|\leq c\}, \quad E_a = \{| \frac{-1}{T^2}\omega_a +\frac{1}{T^2}| \leq c\}.
\]  
For large $T$, the difference between these terms shrinks relative to $c$, causing $\Pr(E_i)$ to approach $\Pr(E_a)$, resulting in $\Pr(E_a)^2 < \Pr(E_i)$.


We exclude truncation cases yielding zero and focus on non-truncation events. The double sum over $a, b\in [T]$ splits into four subcases:

Subcase 1: $a = b = i$  
  \[
  \tilde{\gamma}^i_i \tilde{\gamma}^i_i = \Pr(E_i) \cdot \frac{(T - 1)^2}{T^4}.
  \]  
  
Subcase 2 and 3: $a = i$, $b \neq i$ or $a \neq i$, $b = i$ 
  Each term:
  \[
  \tilde{\gamma}^i_i\tilde{\gamma}^i_b=\frac{T-1}{T^2}\cdot( -\frac{1}{T^2}) = -\frac{T-1}{T^4}.
  \]  
  With $2(T-1)$terms, their expectation is:
  \[
  2(T-1)\cdot\Pr(E_i) \cdot(-\frac{T-1}{T^4})=-2\cdot\Pr(E_i)\cdot\frac{(T-1)^2}{T^4}.
  \]

Subcase 4: $a\neq i$, $b\neq i$ 
  Each term:
  \[
  (-\frac{1}{T^2})^2 = \frac{1}{T^4},
  \]  
  with $(T-1)^2$terms, yielding:
  \[
  \Pr(E_a)^2\cdot (T-1)^2 \cdot \frac{1}{T^4}.
  \]

Combining all subcases:
\[
\sum_{a,b} \mathbb{E}[\tilde{\gamma}_a^i \tilde{\gamma}_b^i]
=(\Pr(E_a)^2-\Pr(E_i))\cdot\frac{(T-1)^2}{T^4}.
\]

Thus, the contribution to $P$ for $i = j$ is:
\[
\frac{1}{d} \sum_{i=1}^T \mathbb{E}[\sum_{a,b\in [T]}\tilde{\gamma}_a^i\tilde{\gamma}_b^i] P_{ii}
= (\Pr(E_a)^2-\Pr(E_i))\cdot\frac{(T-1)^2}{T^4}\cdot \frac{1}{d}\sum_{i=1}^T P_{ii}.
\]

Assuming each $x_i$ follows a zero-mean Gaussian,
\[
P_{ii} = x_i^\top W_V^\top W_F^\top W_F W_V x_i > 0,
\]
since $W_V^\top W_F^\top W_F W_V$ is positive definite and $x_i\ne 0$ almost surely.

Given $\Pr(E_a)^2<\Pr(E_i)$(shown earlier), this term is negative:
\[
\frac{1}{d}\sum_{i = j} \mathbb{E}[\sum_{a,b\in[T]} \tilde{\gamma}_a^i\tilde{\gamma}_b^j]P_{ij} < 0.
\]

\paragraph{Case 2: $i \neq j$}

In this case, we compute the expectation $ \mathbb{E}[{P}_{ij}] $:

\[
\mathbb{E}[{P}_{ij}] = \mathbb{E}[x_i^\top W_V^\top W_F^\top W_F W_V x_j] = \mathbb{E}[ \mathrm{Tr}(x_i^\top W_V^\top W_F^\top W_F W_V x_j )] = \mathrm{Tr}(W_V^\top W_F^\top W_F W_V\mathbb{E}[x_j x_i^\top] ).
\]

Since $x_i$ and $x_j$ are independent, their joint expectation factors as:

\[
\mathbb{E}[x_jx_i^\top ] = \mathbb{E}[x_j ] \mathbb{E}[x_i^\top].
\]

For zero-mean inputs, this becomes:

\[
\mathbb{E}[x_j x_i^\top ] = 0 \quad \Rightarrow \quad \mathbb{E}[{P}_{ij}] = 0.
\]

Thus, in the double sum for $ i \neq j $:

\[
\frac{1}{d} \sum_{i \neq j} \mathbb{E} [ \sum_{a,b \in [T]} \tilde{\gamma}_a^i \tilde{\gamma}_b^j]P_{ij},
\]

all terms vanish, as they depend linearly on $\mathbb{E}[ x_i^\top P x_j ] = 0 $ under independence and zero-mean conditions. Only diagonal terms ($ i = j $) contribute to the approximate gradient, with the $i\neq j$ contribution exactly zero.

Combining Case 1 and Case 2 yields $P<0$, confirming its negativity.

Next, we derive the change in the stable rank of the parameter matrix \( \mathbf{W}_K \). We further assume that \( \mathbf{v}_l \) (\( l \neq 1 \)) is not aligned with any of the non-dominant features, meaning it has the same angle with each of them. Thus, \( \mathbf{v}_l^\top \boldsymbol{\beta}_i = \frac{1}{\sqrt{d-1}} \).

Therefore, due to Equation~\ref{equation:delta_sigma}, we have:
\[
\triangle \sigma_1^2 \approx -\eta{P} \mathcal{O}(\sigma_1^2\mu_1^2\phi^2)\mu_1^2 \phi^2
\]
\begin{align*}
    \triangle \sigma_l^2 &= -\eta{P} \mathcal{O}(\sigma_1^2\mu_1^2\phi^2)\sum_{t=2}^d\mu_t^2 {\left(\mathbf{v}_l^\top\boldsymbol{\beta}_t\right)}^2 \\
    &= -\eta{P} \mathcal{O}(\sigma_1^2\mu_1^2\phi^2) \frac{1}{d-1}\sum_{t=2}^d\mu_t^2 \\
    &= -\eta{P} \mathcal{O}(\sigma_1^2\mu_1^2\phi^2) \frac{\mu_1^2}{d-1}\left( \operatorname{SR}(\mathbf{Z})-1\right)
\end{align*}

We define \( \operatorname{SR}(\mathbf{W}) = \frac{A}{B} \), where \( A = \sum_{i=1}^d \sigma_i^2 \) and \( B = \sigma_1^2 \). After the update, \( \text{SR}(W^\prime) = \frac{A + \delta A}{B + \delta B} \), where \( \delta A = -\eta{P} \mathcal{O}(\sigma_1^2\mu_1^2\phi^2) \mu_1^2\left[ {\phi}^2+\operatorname{SR}(\mathbf{Z})-1\right] \) and \( \delta B =-\eta{P} \mathcal{O}(\sigma_1^2\mu_1^2\phi^2) \mu_1^2  {\phi}^2\).

Thus we have
\begin{align*}
    \triangle \operatorname{SR}(W) &= \operatorname{SR}(W^\prime) - \operatorname{SR}(W) \\
    &=\frac{A + \delta A}{B + \delta B} - \frac{A}{B} \\
    &= (A + \delta A)(\frac{1}{B}-\frac{\delta B}{B^2}) \\
    &\approx \frac{\delta A}{B} - \frac{A\delta B}{B^2} \\
    &= -\eta{P}\mathcal{O}(\mu_1^2\phi^2)\mu_1^2\mathbf{R},
\end{align*}
where $\mathbf{R} = \left[ \left( 1 - \operatorname{SR}(\mathbf{W}_K) \right) \phi^2 + \left( \operatorname{SR}(\mathbf{Z}) - 1 \right) \right].$

So when the following condition is satisfied:
\begin{equation}
    \operatorname{SR}(\mathbf{W}_K)>1+\frac{\operatorname{SR}(\mathbf{Z})-1}{{\phi}^2},
\end{equation}
the $\operatorname{SR}(\mathbf{W}_K)$ will decrease with the update.

\end{proof}

\subsection{Amplification of Singularity Alignment.}
\label{appendix:alignment}
\begin{theorem}[Amplification of Singularity Alignment]
Let \( \mathbf{W}_{QK} \) be updated using the gradient \( \nabla_{\mathbf{W}_{QK}} \mathcal{J}\) as in Equation \ref{eq:qkgrad-simple}. The singular alignment \( \phi \) between \( \mathbf{W}_{QK} \) and \(\mathbf{Z} \) increases after the update. The change in alignment is given by:
\begin{equation}
\triangle \phi = -\eta P \mathcal{O}(\sigma_1^2 \mu_1^2\phi^2) \mu_1^2 \phi \sum_{k \neq 1} \frac{(v_k^\top \beta_1)^2}{\sigma_1^2 - \sigma_k^2},
\end{equation}
indicating a positive update of \( \phi \).
\end{theorem}

\begin{proof}
    \( W_{QK} \) is updated with \( (\Delta W_{QK}) = -\eta \nabla_{\mathbf{W}_{QK}} J \). When the matrix perturbation is sufficiently small, according to Rayleigh–Schrödinger perturbation theory, the change in the new singular vector of \( W_{K} \) is given by:
\[
\mathbf{v}_i^\prime = \mathbf{v}_i + \sum_{k \neq i} \frac{\mathbf{v}_k^\top (\Delta \mathbf{W}_{QK}) \mathbf{v}_i}{\sigma_i^2 - \sigma_k^2}  \, \mathbf{v}_k + O(\|\Delta \mathbf{W}_{QK}\|^2)
\]
So we have:
\begin{align*}
    \mathbf{v}_1^\prime &= \mathbf{v}_1 + \sum_{k \neq 1} \frac{\mathbf{v}_k^\top (\Delta \mathbf{W}_{QK}) \mathbf{v}_1}{\sigma_1^2 - \sigma_k^2}  \mathbf{v}_k + \mathcal{O}(\|\Delta \mathbf{W}_{QK}\|^2) \\
    &= \mathbf{v}_1 + (-\eta P \mathcal{O}(\sigma_1^2\mu_1^2\phi^2))\sum_{k \neq 1} \frac{\mathbf{v}_k^\top (\sum_{t=1}^d\mu_t^2\beta_t\beta_t^\top) \mathbf{v}_1}{\sigma_1^2 - \sigma_k^2}  \mathbf{v}_k + \mathcal{O}(\|\Delta \mathbf{W}_{QK}\|^2)
\end{align*}
So the new singularity ailgnment $\phi^\prime$ is:
\begin{align*}
    \phi^\prime = \beta_1^\top\mathbf{v}_1^\prime &= \beta_1^\top\mathbf{v}_1 + (-\eta P \mathcal{O}(\sigma_1^2\mu_1^2\phi^2))\sum_{k \neq 1} \frac{\mathbf{v}_k^\top (\sum_{t=1}^d\mu_t^2\beta_t\beta_t^\top) \mathbf{v}_1}{\sigma_1^2 - \sigma_k^2}  \beta_1^\top\mathbf{v}_k + \mathcal{O}(\|\Delta \mathbf{W}_{QK}\|^2) \\
    &= \beta_1^\top\mathbf{v}_1 + (-\eta P \mathcal{O}(\sigma_1^2\mu_1^2\phi^2))\sum_{t=1}^d \mu_t^2\sum_{k \neq 1}\frac{(\mathbf{v}_k^\top \beta_t)(\beta_t^\top\mathbf{v}_1)}{\sigma_1^2 - \sigma_k^2} \beta_1^\top\mathbf{v}_k+\mathcal{O}(\|\Delta \mathbf{W}_{QK}\|^2) 
\end{align*}
Due to the Strong Singularity Alignment assumption, we have:
\begin{align*}
    \phi^\prime = \beta_1^\top\mathbf{v}_1 + (-\eta P \mathcal{O}(\sigma_1^2\mu_1^2\phi^2))\left[ \mu_1^2\phi\sum_{k \neq 1}\frac{{(\mathbf{v}_k^\top \beta_1)}^2}{\sigma_1^2 - \sigma_k^2}+ \mathcal{O}(\epsilon)\right] + \mathcal{O}(\|\Delta \mathbf{W}_{QK}\|^2)
\end{align*}
Ignoring \(\mathcal{O}(\epsilon)\) and \(\mathcal{O}(\|\Delta \mathbf{W}_{QK}\|^2)\), we obtain:  
\[
\triangle \phi \approx -\eta {P} \sigma_1^2 \mathcal{O}(\mu_1^2\phi^2) \mu_1^2 \phi \sum_{k \neq 1} \frac{(v_k^\top \beta_1)^2}{\sigma_1^2 - \sigma_k^2}
\]
\end{proof}

\subsection{Amplification of Representation Singularity}
\label{appendix:feature_singularity}
\begin{theorem}[Amplification of Representation Singularity]
We study the Key-Query inner product in a one-layer Transformer: \(\mathbf{Z}_K^\top \mathbf{Z}_Q= \mathbf{Z} \mathbf{W}_{QK} \mathbf{Z}^\top \), then the stable rank of \( \mathbf{Z}_K \) and \( \mathbf{Z}_Q \) is:
\begin{equation}
    \operatorname{SR}(\mathbf{Z}_K) = \operatorname{SR}(\mathbf{Z}_Q) \approx 1+\frac{\left[ \operatorname{SR}(\mathbf{W}_K)-1\right]\left[ \operatorname{SR}(\mathbf{Z})-1 \right]}{(d-1)\phi^2}
\end{equation}
Let \( \mathbf{W}_{QK} \) be updated using the simplified gradient \( \nabla_{W_{QK}} \mathcal{J}\) as in Equation \ref{eq:qkgrad-simple}. Then if the following condition being satisfied: $\operatorname{SR}(\mathbf{W}_K) >  1 + \frac{\phi^2\left[ \operatorname{SR}(\mathbf{Z})-1 \right]}{{(d-1)}^2},$ $\operatorname{SR}(\mathbf{Z}_K)$ decreases.
\end{theorem}

\begin{proof}
    We first derive the form of the Key-Query inner product \(\mathbf{Z}_K^\top \mathbf{Z}_Q\), thereby obtaining the singular vectors and singular values.
\begin{align*}
        \mathbf{Z}_K^\top \mathbf{Z}_Q &= \mathbf{Z}\mathbf{W}_{QK}\mathbf{Z}^\top \\
        &= (\sum_{i=1}^d\mu_i\boldsymbol{\alpha}_i\boldsymbol{\beta}_i^\top)(\sum_{t=1}^d\sigma_t^2\mathbf{v}_t\mathbf{v}_t^\top)(\sum_{j=1}^d\mu_j\boldsymbol{\beta}_j\boldsymbol{\alpha}_j^\top)\\
        &=\sum_{t=1}^d\sigma_t^2\sum_{i,j}\mu_i\mu_j(\boldsymbol{\alpha}_i\boldsymbol{\beta}_i^\top\mathbf{v}_t\mathbf{v}_t^\top\boldsymbol{\beta}_i\boldsymbol{\alpha}_i^\top)\\
        &= \sigma_1^2\sum_{i,j}(\boldsymbol{\alpha}_i\boldsymbol{\beta}_i^\top\mathbf{v}_1\mathbf{v}_1^\top\boldsymbol{\beta}_i\boldsymbol{\alpha}_i^\top) + \sum_{t=2}^d\sigma_t^2\sum_{i,j}\mu_i\mu_j(\boldsymbol{\alpha}_i\boldsymbol{\beta}_i^\top\mathbf{v}_t\mathbf{v}_t^\top\boldsymbol{\beta}_i\boldsymbol{\alpha}_i^\top) 
\end{align*}
We further assume that \( \mathbf{v}_l \) (\( l \neq 1 \)) is not aligned with any of the non-dominant features, meaning it has the same angle with each of them. Thus, \( \mathbf{v}_l^\top \boldsymbol{\beta}_i = \frac{1}{\sqrt{d-1}} \). And due to Strong Singularity Alignment Assumption, we have:
\begin{align*}
    \mathbf{Z}_K^\top \mathbf{Z}_Q &= \sigma_1^2\mu_1^2\phi^2\boldsymbol{\alpha}_1\boldsymbol{\alpha}_1^\top + (\sum_{t=2}^d\sigma_t^2)\frac{1}{d-1}\sum_{i,j}\mu_i\mu_j\boldsymbol{\alpha}_i\boldsymbol{\alpha}_j^\top + \mathcal{O}(\epsilon) + \mathcal{O}(\epsilon^2) \\
    &= \sigma_1^2\mu_1^2\phi^2\boldsymbol{\alpha}_1\boldsymbol{\alpha}_1^\top + (\operatorname{SR}(\mathbf{W}_K)-1)\sigma_1^2\frac{1}{d-1}\sum_{i,j}\mu_i\mu_j\boldsymbol{\alpha}_i\boldsymbol{\alpha}_j^\top + \mathcal{O}(\epsilon) + \mathcal{O}(\epsilon^2)
\end{align*}

So we get the approximated singular value and singular vectors of $\mathbf{Z}_K$ and $\mathbf{Z}_Q$:
\[
\sigma_1^2(\mathbf{Z}_K) \approx \sigma_1^2\mu_1^2\phi^2,
\]
with corresponding singular vector:$\mathbf{v}_1(\mathbf{Z}_K) = \boldsymbol{\alpha_1}$, and
\[
\sigma_2^2(\mathbf{Z}_K) \approx (\operatorname{SR}(\mathbf{W}_K)-1)\sigma_1^2\frac{1}{d-1}(\operatorname{SR}(\mathbf{Z})-1)\mu_1^2,
\]
with corresponding singular vector:$\mathbf{v}_2(\mathbf{Z}_K) = \frac{1}{\sqrt{\sum_{t=2}^d\mu_i^2}}\sum_{i=2}^d \boldsymbol{\alpha}_i.$
The remaining singular values and singular vectors are primarily contributed by \(\mathcal{O}(\epsilon) + \mathcal{O}(\epsilon^2)\), which can be neglected.

So the Stable rank of $\mathbf{Z}_K$ and $\mathbf{Z}_Q$ can be approximated as:
\begin{align*}
    \operatorname{SR}(\mathbf{Z}_K) &\approx \frac{\sigma_1^2\mu_1^2\phi^2 +(\operatorname{SR}(\mathbf{W}_K)-1)\sigma_1^2\frac{1}{d-1}(\operatorname{SR}(\mathbf{Z})-1)\mu_1^2 }{\sigma_1^2\mu_1^2\phi^2} \\
    &= 1+\frac{\left[ \operatorname{SR}(W_K)-1\right]\left[ \operatorname{SR}(Z)-1 \right]}{(d-1)\phi^2}
\end{align*}

Then, we derive the changes in the singular values of \(  \mathbf{Z}_K^\top \mathbf{Z}_Q \) after the update of \( W_{QK} \).
\begin{align*}
    \triangle \mathbf{Z}_K^\top \mathbf{Z}_Q &= \mathbf{Z}(-\eta \nabla_{\mathbf{W}_{QK}} \mathcal{J})\mathbf{Z}^\top \\
    &= (-\eta P\mathcal{O}(\sigma_1^2\mu_1^2\phi^2))(\sum_{i=1}^d\mu_i\boldsymbol{\alpha}_i\boldsymbol{\beta}_i^\top)(\sum_{t=1}^d\mu_t^2\boldsymbol{\beta}_t\boldsymbol{\beta}_t^\top)(\sum_{j=1}^d\mu_j\boldsymbol{\beta}_j\boldsymbol{\alpha}_j^\top)\\
    &= (-\eta P\mathcal{O}(\sigma_1^2\mu_1^2\phi^2))\sum_{t=1}^d\mu_t^2\sum_{i,j}\mu_i\mu_j(\boldsymbol{\alpha}_i\boldsymbol{\beta}_i^\top\boldsymbol{\beta}_t\boldsymbol{\beta}_t^\top\boldsymbol{\beta}_i\boldsymbol{\alpha}_i^\top) \\
    &= (-\eta P\mathcal{O}(\sigma_1^2\mu_1^2\phi^2))\sum_{t=1}^d\mu_t^4 \boldsymbol{\alpha}_t\boldsymbol{\alpha}_t^\top
\end{align*}
So the change of the square of singular values of $\mathbf{Z}_K^\top \mathbf{Z}_Q$ is:
\begin{align*}
    \triangle \sigma_1^2(\mathbf{Z}_K) &= \mathbf{v}_1^\top(\mathbf{\mathbf{Z}_K})\triangle \mathbf{Z}_K^\top \mathbf{Z}_Q\mathbf{v}_1(\mathbf{\mathbf{Z}_K}) \\
    &= (-\eta P\mathcal{O}(\sigma_1^2\mu_1^2\phi^2))\boldsymbol{\alpha}_1^\top \sum_{t=1}^d\mu_t^4 \boldsymbol{\alpha}_t\boldsymbol{\alpha}_t^\top \boldsymbol{\alpha}_1 \\
    &= (-\eta P\mathcal{O}(\sigma_1^2\mu_1^2\phi^2))\mu_1^4
\end{align*}
\begin{align*}
    \triangle \sigma_2^2(\mathbf{Z}_K) &= \mathbf{v}_2^\top(\mathbf{\mathbf{Z}_K})\triangle \mathbf{Z}_K^\top \mathbf{Z}_Q\mathbf{v}_2(\mathbf{\mathbf{Z}_K}) \\
    &= (\eta P\mathcal{O}(\sigma_1^2\mu_1^2\phi^2)) {(\frac{1}{\sqrt{\sum_{t=2}^d\mu_i^2}}\sum_{i=2}^d \boldsymbol{\alpha}_i)}^\top (\sum_{t=1}^d\mu_t^4 \boldsymbol{\alpha}_t\boldsymbol{\alpha}_t^\top)(\frac{1}{\sqrt{\sum_{t=2}^d\mu_i^2}}\sum_{i=2}^d \boldsymbol{\alpha}_i) \\
    &= (\eta P\mathcal{O}(\sigma_1^2\mu_1^2\phi^2)) \frac{\sum_{i=2}^d\mu_i^6}{(\operatorname{SR}(\mathbf{Z})-1)\mu_1^2} \\
    &\approx (\eta P\mathcal{O}(\sigma_1^2\mu_1^2\phi^2)) \frac{(\operatorname{SR}(\mathbf{Z})-1)^2\mu_1^4}{(d-1)^3}
\end{align*}

So if $\frac{\triangle \sigma_1^2(\mathbf{Z}_K) + \triangle \sigma_2^2(\mathbf{Z}_K)}{\triangle \sigma_1^2(\mathbf{Z}_K)} < \operatorname{SR}(\mathbf{Z}_K)$, which is that
\[
\operatorname{SR}(\mathbf{W}_K) >  1 + \frac{\phi^2\left[ \operatorname{SR}(\mathbf{Z})-1 \right]}{{(d-1)}^2},
\]
$\operatorname{SR}(\mathbf{Z}_K)$ decreases.

\end{proof}

\subsection{Lower and Upper bound of gradient’s Frobenius norm}
\label{appendix:gradient_upperbound}
\begin{theorem}[Bounds on Gradient Frobenius Norm] 
The Frobenius norm of the \(QK\)-gradient \(\nabla_{\mathbf{W}_{QK}} \mathcal{J}\) is bounded as follows:

1. Lower Bound: Assuming a fixed Frobenius norm \(\|\mathbf{W}_{QK}\|_{\text{F}}^2 = \sum_{k=1}^d \sigma_k^2 = M > 0\) at a given step, the gradient norm satisfies:
\begin{equation}
\label{eq:lower-bound}
    \|\nabla_{\mathbf{W}_{QK}} \mathcal{J}\|_{\text{F}} \geq K \alpha_1 \sigma_1^2(\mathbf{W}_{QK}),
\end{equation}
where \( K = P \cdot \| {\sum_{t=1}^d\mu_t^2\boldsymbol{\beta}_t\boldsymbol{\beta}_t^\top} \| > 0 \) and  \( \alpha_1 =\left( \sum_{r=1}^d \mu_r \cdot (\mathbf{v}_1^\top \boldsymbol{\beta}_r) \right)^2. \) 

2. Upper Bound:
\begin{equation}
\label{eq:upper-bound}
     \|\nabla_{\mathbf{W}_{QK}} \mathcal{J}\|_{\text{F}} \leq C \sigma_{1}^2(\mathbf{W}_V){\left[ \sigma_1(\mathbf{W}_{F_1})\sigma_1(\mathbf{W}_{F_2})+1\right]}^2\sigma_1(\mathbf{W}_{QK}),
\end{equation}
where $C = \frac{1}{d}\sum_{i,j,a,b \in [T]}\mathbb{E}\left[ \tilde{\gamma}_a^i\tilde{\gamma}_b^j\right]\mathcal{O}(\mu_1^4\phi^2) \operatorname{SR}(\mathbf{Z}).$ \
\end{theorem}

\begin{proof}
    In the simpilified single layer transformer, The gradient of \( \mathbf{W}_{QK} \) can be approximated as:
\[
    \nabla_{\mathbf{W}_{QK}} \mathcal{J} \approx {P}  \mathcal{O}(\sigma_1^2\mu_1^2\phi^2) \sum_{t=1} \mu_t^2 \boldsymbol{\beta}_t \boldsymbol{\beta}_t^\top,
\]
where \( P = \frac{1}{d} \sum_{i,j,a,b \in [T]} \mathbb{E} \left[ \tilde{\gamma}^i_a \tilde{\gamma}^j_b {P}_{ij} \right] \).

We first bound $P$:
\begin{align*}
    P &= \frac{1}{d} \sum_{i,j,a,b \in [T]} \mathbb{E} \left[ \tilde{\gamma}^i_a \tilde{\gamma}^j_b {P}_{ij} \right] \\
    &= \frac{1}{d} \sum_{i,j,a,b \in [T]} \mathbb{E} \left[ \tilde{\gamma}^i_a \tilde{\gamma}^j_b \mathbf{x}_i^\top \mathbf{W}_V^\top \mathbf{W}_F^\top \mathbf{W}_F\mathbf{W}_V \mathbf{x}_j\right]\\
    &\leq \frac{1}{d} \sum_{i,j,a,b \in [T]} \mathbb{E} \left[ {\gamma}^i_a {\gamma}^j_b\right]\sigma_1^2(\mathbf{W}_V){\left[ \sigma_1(\mathbf{W}_{F_1})\sigma_1(\mathbf{W}_{F_2})+1\right]}^2
\end{align*}

Then we bound $\sum_{t=1} \mu_t^2 \boldsymbol{\beta}_t \boldsymbol{\beta}_t^\top$:
\begin{align*}
    \|\sum_{t=1} \mu_t^2 \boldsymbol{\beta}_t \boldsymbol{\beta}_t^\top \|_F &\leq \sum_{t=1}\mu_t^2\|\beta_t\|^2 \\
    &\leq \sum_{t=1}^d \mu_t^2
\end{align*}

Thus the F-Norm of $QK$-gradient $\nabla_{\mathbf{W}_{QK}} {J}$ is upper bounded by:
\[
\|\nabla_{\mathbf{W}_{QK}} {J} \|_F \leq C \sigma_{1}^2(\mathbf{W}_V){\left[ \sigma_1(\mathbf{W}_{F_1})\sigma_1(\mathbf{W}_{F_2})+1\right]}^2\sigma_1(\mathbf{W_{QK}}),
\]
where $C = \frac{1}{d}\sum_{i,j,a,b \in [T]}\mathbb{E}\left[ \gamma_a^i\gamma_b^j\right]\mathcal{O}(\mu_1^4\phi^2) \operatorname{SR}(\mathbf{Z}).$

Through equation~\ref{equation:gradient_approx}, the gradient can be approximated as 
\begin{equation}
    \nabla_{\mathbf{W}_{QK}} \mathcal{J} = {P} \left[\sum_{k=1}^d\sigma_k^2\sum_{r,s=1}^d\mu_r\mu_s\left( \boldsymbol{\beta}_r^\top \mathbf{v}_k\right)\left( \mathbf{v}_k^\top \boldsymbol{\beta}_s\right)\sum_{t=1}^d\mu_t^2\boldsymbol{\beta}_t\boldsymbol{\beta}_t^\top\right].
\end{equation}

Rearranging terms and defining
\[
\alpha_k := \left( \sum_{r=1}^d \mu_r \, \mathbf{v}_k^\top \boldsymbol{\beta}_r \right)^2
= (\mathbf{v}_k^\top \mathbf{p})^2,
\quad \text{where} \quad \mathbf{p} := \sum_{r=1}^d \mu_r \, \boldsymbol{\beta}_r,
\]
we rewrite the gradient as
\[
\nabla_{\mathbf{W}_{QK}} \mathcal{J}
\;=\;
P \left[
\sum_{k=1}^d \sigma_k^2 \alpha_k
\cdot
\sum_{t=1}^d \mu_t^2 \boldsymbol{\beta}_t \boldsymbol{\beta}_t^\top
\right].
\]

Gradient norm depends only on \(\sum_k \sigma_k^2 \alpha_k\): Observe that both \(P\) and \(\sum_{t=1}^d \mu_t^2 \boldsymbol{\beta}_t \boldsymbol{\beta}_t^\top\) are independent of the matrix \(\mathbf{W}_{QK}\)'s singular values \(\{\sigma_k^2\}\). Therefore, the gradient magnitude (under Frobenius or spectral norm) depends linearly on the scalar quantity
\[
\sum_{k=1}^d \sigma_k^2 \alpha_k.
\]

Lower bound on \( \alpha_1 \).

By definition:
\[
\mathbf{v}_1^\top \mathbf{p}
= \sum_{r=1}^d \mu_r\,(\mathbf{v}_1^\top \boldsymbol{\beta}_r)
= \mu_1\,(\mathbf{v}_1^\top \boldsymbol{\beta}_1)
+ \sum_{r=2}^d \mu_r\,(\mathbf{v}_1^\top \boldsymbol{\beta}_r).
\]

Using Assumption:
- \( \mathbf{v}_1^\top \boldsymbol{\beta}_1 = \sqrt{1 - \epsilon} \),
- and for \( r \ge 2 \), \( |\mathbf{v}_1^\top \boldsymbol{\beta}_r| \le \sqrt{\epsilon} \),
- also, \( \mu_r \le \mu_1 \) for all \( r \), and \( \sum_{r=2}^d \mu_r \le \frac{1}{\eta}\mu_1 - \mu_1 \le \left(\frac{1}{\eta} - 1\right)\mu_1 \).

Therefore:
\[
|\sum_{r=2}^d \mu_r (\mathbf{v}_1^\top \boldsymbol{\beta}_r)|
\le \left(\sum_{r=2}^d \mu_r\right) \cdot \sqrt{\epsilon}
\le \left( \frac{1}{\eta} - 1 \right) \mu_1 \sqrt{\epsilon}.
\]

Hence:
\[
\mathbf{v}_1^\top \mathbf{p}
\ge \mu_1 \sqrt{1 - \epsilon} - \left( \frac{1}{\eta} - 1 \right) \mu_1 \sqrt{\epsilon}
= \mu_1 \left( \sqrt{1 - \epsilon} - \left( \frac{1}{\eta} - 1 \right) \sqrt{\epsilon} \right).
\]

Let us denote:
\[
L := \mu_1 \left( \sqrt{1 - \epsilon} - \left( \frac{1}{\eta} - 1 \right) \sqrt{\epsilon} \right).
\]
Then,
\[
\alpha_1 = (\mathbf{v}_1^\top \mathbf{p})^2 \ge L^2.
\]

Upper bound on \( \alpha_k \) for \( k > 1 \)

Similarly, for \( k > 1 \), we write:
\[
\alpha_k = (\mathbf{v}_k^\top \mathbf{p})^2
= \left( \sum_{r=1}^d \mu_r\,(\mathbf{v}_k^\top \boldsymbol{\beta}_r) \right)^2.
\]

From the orthonormality of \( \{ \mathbf{v}_k \} \) and \( \{ \boldsymbol{\beta}_r \} \), we have:
\[
\sum_{k=1}^d (\mathbf{v}_k^\top \boldsymbol{\beta}_1)^2 = 1.
\]

Since \( \mathbf{v}_1^\top \boldsymbol{\beta}_1 = \sqrt{1 - \epsilon} \), we get:
\[
\sum_{k > 1} (\mathbf{v}_k^\top \boldsymbol{\beta}_1)^2 \le \epsilon,
\quad \Rightarrow \quad
\max_{k > 1} |\mathbf{v}_k^\top \boldsymbol{\beta}_1| \le \sqrt{\epsilon}.
\]

Also, \( |\mathbf{v}_k^\top \boldsymbol{\beta}_r| \le 1 \) for all \( r \). So we can bound:
\[
|\mathbf{v}_k^\top \mathbf{p}| 
= \left| \sum_{r=1}^d \mu_r (\mathbf{v}_k^\top \boldsymbol{\beta}_r) \right|
\le \mu_1 \left( \sqrt{\epsilon} + \left( \frac{1}{\eta} - 1 \right) \right).
\]

Define:
\[
U := \mu_1 \left( \sqrt{\epsilon} + \left( \frac{1}{\eta} - 1 \right) \right),
\quad \Rightarrow \quad
\alpha_k \le U^2.
\]

We now compare the bounds:
\[
\alpha_1 \ge L^2,
\quad
\alpha_k \le U^2 \quad (k > 1).
\]

Thus, to ensure \( \alpha_1 > \alpha_k \) for all \( k > 1 \), it suffices that
\[
L^2 > U^2,
\quad \text{or equivalently} \quad
\sqrt{1 - \epsilon} - \left( \frac{1}{\eta} - 1 \right) \sqrt{\epsilon}
>
\sqrt{\epsilon} + \left( \frac{1}{\eta} - 1 \right).
\]

This inequality clearly holds when: \( \epsilon \ll 1 \), \( \eta \approx 1 \) (i.e., \(\mu_1\) is strongly dominant), and \( d \) is moderate.

Then we will prove Gradient norm increases when more energy is concentrated on $\sigma_1$.

Now, define:
\[
S := \sum_{k=1}^d \sigma_k^2 \alpha_k.
\]

We will now prove that \( S \) increases as \(\sigma_1^2\) increases, i.e., when more energy is concentrated on the top singular value.

Let \(\boldsymbol{\sigma}^2 := (\sigma_1^2, \dots, \sigma_d^2)\) be any distribution of energy with \(\sum_k \sigma_k^2 = C\)

Define an energy transfer operation: for small \( \varepsilon > 0 \), move energy from \( \sigma_j^2 \) to \( \sigma_1^2 \), i.e.,
\[
\sigma_1^{2\,\prime} = \sigma_1^2 + \varepsilon, \quad \sigma_j^{2\,\prime} = \sigma_j^2 - \varepsilon,
\quad \text{with all other } \sigma_k^2 \text{ unchanged.}
\]

This preserves total energy:
\[
\sum_{k=1}^d \sigma_k^{2\,\prime} = C.
\]

Then, the corresponding change in \(S\) is:
\[
S^\prime - S = \varepsilon(\alpha_1 - \alpha_j).
\]

But since \(\alpha_1 > \alpha_j\), we have \(S^\prime > S\). Hence, moving energy from lower-weighted singular directions to the top singular direction strictly increases the gradient norm.

Maximizer of \(S\) under \(\sum_k \sigma_k^2 = C\):

Because \( \alpha_1 > \alpha_2 \ge \dots \ge \alpha_d \), the linear objective \( \sum_k \sigma_k^2 \alpha_k \) is maximized under the constraint \( \sum_k \sigma_k^2 = C \) when:
\[
\sigma_1^2 = C,
\quad
\sigma_k^2 = 0 \text{ for } k > 1.
\]

Due to $\|\nabla_{\mathbf{W}_{QK}} \mathcal{J}\| = K \cdot \sum_{k=1}^d \sigma_k^2 \alpha_k.$ and \(\alpha_1 > \alpha_k\) and all \(\sigma_k^2 \ge 0\), we have:
\[
\sum_{k=1}^d \sigma_k^2 \alpha_k \ge \sigma_1^2 \alpha_1.
\]

Thus,
\[
\|\nabla_{\mathbf{W}_{QK}} \mathcal{J}\|
\;\ge\;
K \cdot \sum_k \sigma_k^2 \alpha_k
\;\ge\;
K \cdot \alpha_1 \cdot \sigma_1^2.
\]
\end{proof}

\section{Empirical Validation of Theorem Robustness}
\label{appendix:empirical_validation}
The Strong Singularity Alignment assumption—where the top singular vector $v_1$ of the parameter matrix aligns closely with the dominant semantic basis vector $\beta_1$, i.e., $v_1^\top \beta_1 = \sqrt{1 - \epsilon}$ with $\epsilon \to 0$—simplifies our theoretical analysis. It holds when this alignment significantly exceeds other vector pairings, a condition met across practical regimes. While Fig~\ref{figure:metrices} shows alignment near 0.8, not 1, this remains sufficient for our results, as detailed below.  

The assumption underpins two key analyses: 

1. Gradient Approximation Bounds: Under the given assumptions, we decompose:
$\sum_{k=1}^d \sigma_k^2 \sum_{r,s=1}^d \mu_r \mu_s (\beta_r^\top \mathbf{v}_k)(\mathbf{v}_k^\top \beta_s)
= \sigma_1^2 \mu_1^2 (\mathbf{v}_1^\top \beta_1)^2 + 2 \sigma_1^2 \sum_{r=2}^d \mu_1 \mu_r (\beta_1^\top \mathbf{v}_1)(\mathbf{v}_1^\top \beta_r) $ $+ \sigma_1^2 \sum_{r,s=2}^d \mu_r \mu_s (\beta_r^\top \mathbf{v}_1)(\mathbf{v}_1^\top \beta_s) + \sum_{k=2}^d \sigma_k^2 \sum_{r,s=1}^d \mu_r \mu_s (\beta_r^\top \mathbf{v}_k)(\mathbf{v}_k^\top \beta_s).$

This is bounded as:
$\leq \sigma_1^2 \mu_1^2 \left[ (1 - \epsilon) + 2 \frac{\epsilon}{d-1} \frac{1}{\operatorname{SR}(\mathbf{Z})} + \frac{\epsilon}{d-1} \frac{1}{\operatorname{SR}(\mathbf{Z})^2} + \frac{1}{d} \frac{1}{\operatorname{SR}(\mathbf{Z})^2} (SR(\mathbf{W}) - 1) \right].$
The first term, $ (1 - \epsilon) $, dominates as $\mu_1$-dominance) and for large $ d $ (e.g., 768 in BERT-base), which suppresses subsequent terms. The exact value of $ \epsilon $ minimally affects the magnitude of subsequent terms, preserving the dominance of $ 1 - \epsilon $. Thus, we approximate the expression as $ \mathcal{O}(\sigma_1^2 \mu_1^2 \phi^2) $, contributing to Equation (3) in Theorem 2.1, with similar bounds like $ \mathcal{O}(\mu_1^2 \phi^2) $ derived analogously.

2. $\sigma_1$ and $SR(W_{KQ})$ Dynamics: 
The shift in squared singular values,
$
\Delta \sigma_l^2 = -\eta \cdot \mathcal{P} \cdot \mathcal{O}(\sigma_1^2 \mu_1^2 \phi^2) \sum_{t=1}^d \mu_t^2 (\mathbf{v}_l^\top \boldsymbol{\beta}_t)^2,
$
drives a simultaneous increase in $\sigma_1$ and decrease in $SR(W_{KQ})$ via the dominant term $\mu_1^2 (\mathbf{v}_1^\top \boldsymbol{\beta}_1)^2$. This dominance stems from: (a) $\mu_1$ exceeding other $\mu_i$ much, and (b) the singularity alignment $v_1^\top \beta_1 = \sqrt{1 - \epsilon}$ greatly exceeding other pairings ($(\mathbf{v}_l^\top \boldsymbol{\beta}_t)^2, l \neq 1, t \neq 1$). In our experiments (hidden dimension 768), $v_1^\top \beta_1 = 0.8$ ($\epsilon = 0.36$, $(\mathbf{v}_1^\top \boldsymbol{\beta}_1)^2 = 0.64$) yields an expected $(\mathbf{v}_1^\top \boldsymbol{\beta}_i)^2$ or $(\mathbf{v}_i^\top \boldsymbol{\beta}_1)^2$ ($i \neq 1$) of $3.8 \times 10^{-4}$—three orders smaller—ensuring $\Delta \sigma_1^2$ overshadows $\sum_{l \neq 1} \Delta \sigma_l^2$, sustaining this dynamic.

\section{Method Details}

\subsection{Smoothing Policy}
\label{appendix:smoothing-policy}
Given a weight matrix \(\mathbf{W}\) and its dominant singular values $\bm{\sigma}=\{\sigma_i\}_{i=1}^{\lceil \op{SR}(\mathbf{W}) \rceil}$, we introduce several smoothing functions below.

\textbf{Convolution.} 
We apply a convolution kernel \(\mathbf{k} = \{k_j\}_{j=-m}^{m}\) to the dominant singular values \(\bm{\sigma}\) to smooth the transition between them. The smoothed singular values \(\bm{\sigma}^{\text{smooth}} = \{\sigma_i^{\text{smooth}}\}\) are computed by convolving the kernel with the original singular values:
\[
\sigma_i^{\text{smooth}} = \sum_{j=-m}^{m} k_j \cdot \sigma_{i+j}, \quad \forall i \in \{m+1, \dots, \lceil \op{SR}(\mathbf{W}) \rceil - m\}
\]
where \(\sigma_{i+j}\) refers to neighboring singular values, and the kernel \(\mathbf{k}\) is typically normalized to ensure that the sum of the kernel elements equals 1:
\[
\sum_{j=-m}^{m} k_j = 1.
\]
This convolution operation smooths the singular value spectrum by averaging over a local neighborhood of values. It reduces abrupt changes between adjacent singular values while preserving their relative order, ensuring that the most dominant singular values retain their influence in the parametric space.

\textbf{Softmax.}
We apply the Softmax function to the dominant singular values \(\bm{\sigma} = \{\sigma_i\}_{i=1}^{\lceil \op{SR}(\mathbf{W}) \rceil}\) to ensure that the singular values are normalized while retaining their relative importance. The Softmax-smoothed singular values \(\bm{\sigma}^{\text{softmax}} = \{\sigma_i^{\text{softmax}}\}\) are computed as follows:

\[
\sigma_i^{\text{softmax}} = \frac{\exp(\beta \cdot \sigma_i)}{\sum_{j=1}^{\lceil \op{SR}(\mathbf{W}) \rceil} \exp(\beta \cdot \sigma_j)}
\]
\[
\sigma_i^{*} = \frac{\exp(\beta \cdot \sigma_i)}{\sum_{j=1}^{n} \exp(\beta \cdot \sigma_j)}
\]

where \(\beta\) is a scaling parameter that controls the sharpness of the distribution. A higher \(\beta\) places more emphasis on the larger singular values, while a lower \(\beta\) spreads the emphasis more evenly across all singular values.

The Softmax function smooths the singular values by normalizing them in a way that maintains the hierarchy of dominance, but it prevents any one singular value from dominating excessively. This ensures that the singular value distribution remains balanced and well-behaved during the smoothing process.

\section{Experiment deatils}
\subsection{Model Configuration Details}
The detailed training configuration of BERT and GPT-2 in shown in Table~\ref{tab:BERT-config-app} and Table~\ref{tab:GPT-config-app}.

\label{appendix:model-configuration}
\begin{table}[H]
\centering
\caption{Configuration of BERT training.}
\begin{tabular}{@{}lcc@{}}
\toprule
Configurations & BERT-base & BERT-large \\ \midrule
Hidden activation       & \multicolumn{2}{c}{GELU} \\
Max position embeddings & \multicolumn{2}{c}{512}  \\
Tokenizer               & \multicolumn{2}{c}{RoBERTa-base}  \\
Vocabulary size         & \multicolumn{2}{c}{50265} \\
Sequence length         & \multicolumn{2}{c}{128} \\
Layernorm $\epsilon$    & \multicolumn{2}{c}{1e-12} \\
Dropout probability     & \multicolumn{2}{c}{0.1} \\ 
Optimizer               & \multicolumn{2}{c}{AdamW} \\ 
AdamW weight decay      & \multicolumn{2}{c}{0.01} \\
AdamW $(\beta_1, \beta_2)$ & \multicolumn{2}{c}{(0.9, 0.999)} \\ 
AdamW $\epsilon$        & \multicolumn{2}{c}{0.01} \\
LR warm-up steps        & \multicolumn{2}{c}{1000} \\
LR schedular            & \multicolumn{2}{c}{Cosine} \\ \midrule
Hidden size          & 768 & 1024 \\ 
Intermediate size    & \multicolumn{2}{c}{4 $\times$ Hidden size} \\ 
Hidden Layers        & 12 & 24 \\ 
Num. attention heads & 12 & 16 \\ 
Batch size           & 64 & 32 \\ \bottomrule
\end{tabular}
\label{tab:BERT-config-app}
\end{table}

\begin{table}[H]
\centering
\caption{Configuration of GPT-2 training.}
\begin{tabular}{@{}lccc@{}}
\toprule
Configurations & GPT-2-medium & GPT-2-large & GPT-2-XL \\ \midrule
Hidden activation       & \multicolumn{3}{c}{GELU} \\
Max position embeddings & \multicolumn{3}{c}{512}  \\
Tokenizer               & \multicolumn{3}{c}{r50-k}  \\
Vocabulary size         & \multicolumn{3}{c}{50257} \\
Sequence length         & \multicolumn{3}{c}{256} \\
Layernorm $\epsilon$    & \multicolumn{3}{c}{1e-5} \\
Dropout probability     & \multicolumn{3}{c}{0.5} \\ 
Optimizer               & \multicolumn{3}{c}{AdamW} \\ 
AdamW weight decay      & \multicolumn{3}{c}{0.1} \\
AdamW $(\beta_1, \beta_2)$ & \multicolumn{3}{c}{(0.9, 0.995)} \\ 
AdamW $\epsilon$        & \multicolumn{3}{c}{0.01} \\
LR warm-up steps        & \multicolumn{3}{c}{1000} \\
LR schedular            & \multicolumn{3}{c}{Linear} \\ \midrule
Hidden size          & 1024 & 1280 & 1600 \\ 
Intermediate size    & \multicolumn{3}{c}{4 $\times$ Hidden size} \\ 
Hidden Layers        & 24 & 36 & 48 \\ 
Num. attention heads & 16 & 20 & 25 \\ 
Batch size           & 64 & 32 & 16  \\ \bottomrule
\end{tabular}
\label{tab:GPT-config-app}
\end{table}

\subsection{Loss Curve with Different Stabilization Methods}
\label{appendix:loss-PSS&others}
Fig.~\ref{figure:loss-curves} illustrates the test loss curves on BERT-base at LRs of 1e-4 and 4e-4. PSS detects and stabilizes loss explosions with minimal spikes and no additional training steps to recover. In contrast, GC and OR fail to prevent instability and result in training failure.  QK-Norm, while stabilizing training, does so at the cost of increased test loss. 
\begin{figure}[H]
    \centering
    \includegraphics[width=.7\linewidth]{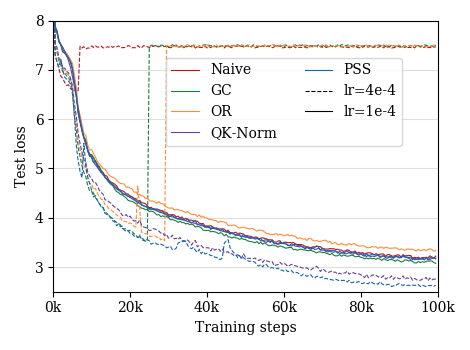} 
     
    \caption{Loss curves with different stabilization methods.}
    \label{figure:loss-curves}
\end{figure}

\subsection{Experiment Results on Larger Models}
\label{appendix:Experiment Results on Larger Models}

We validated the effectiveness of our method on larger models, including BERT-large (340M), GPT-2-Large (774M), and GPT-2-XL (1.5B). For these larger models, we primarily compared with the naive baseline to demonstrate our method's ability to stabilize training (see table~\ref{table:larger-model}). In all three models, we successfully used large learning rates, which would cause instability in the naive baseline, to achieve stable training.

\begin{table}[H]
\centering
\caption{Results on BERT-Large, GPT-2-Large, GPT-2-XL and LLaMa3-8B}
\begin{tabular}{@{}ccccccccc@{}}
\toprule
\multirow{2}{*}{\diagbox{Method}{Model}} &
\multicolumn{2}{c}{BERT-Large} & \multicolumn{2}{c}{GPT-2-Large} & \multicolumn{2}{c}{GPT-2-XL} & \multicolumn{2}{c}{LLaMa3-8B}  \\ 
\cmidrule(l){2-3} \cmidrule(l){4-5} \cmidrule(l){6-7} \cmidrule(l){8-9} 
& lr=5e-5 & lr=1e-4 & lr=5e-5 & lr=1e-4 & lr=1e-5 & lr=5e-5 & lr=5e-4 & lr=1e-3 \\ \midrule
Naive & $0/2_{2.4\pm0.1}$ & $2/2_{-}$ & $0/1_{3.9}$ & $1/1_{-}$ & $0/1_{3.7}$ & $1/1_{-}$ & $0/2_{3.8\pm0.1}$ & $2/2_{-}$ \\ \midrule
GC & $-$ & $0/2_{2.9\pm0.1}$ & $-$ & $0/2_{5.2\pm0.3}$ & $-$ & $0/1_{4.3}$ & $-$ & $-$ \\
OR & $-$ & $3/3_{-}$ & $-$ & $1/1_{-}$ & $-$ & $1/1_{-}$ & $-$ & $-$ \\
QK-Norm & $-$ & $0/2_{2.9\pm0.2}$ & $-$ & $1/1_{-}$ & $-$ & $0/1_{3.9}$ & $-$ & $-$ \\
PSS   & $0/2_{2.4\pm0.1}$ & $0/2_{2.7\pm0.1}$ & $0/1_{4.0}$ & $0/1_{4.2}$ & $0/1_{3.8}$ & $0/1_{3.9}$ & $0/2_{3.8\pm0.1}$ & $0/2_{3.8\pm0.1}$ \\ \bottomrule
\end{tabular}
\label{table:larger-model}
\end{table}

\subsection{Experiment Results on different Hyper-parameters}
\label{appendix:Experiment Results on different Hyper-parameters}
In addition to the learning rate, we also evaluated the performance of the PSS method under various other hyper-parameter settings. Specifically, we compared the PSS method with the Naive method across different batch sizes (set to 1 and 64) and different learning rate warm-up steps (set to 1 and 1000). The experimental results(see table~\ref{table:bs-lrwarmup}) indicate that the Naive method exhibited instability when the batch size was excessively small (batch size = 1) and when no learning rate warm-up strategy was employed. In contrast, the PSS method demonstrated the capability to maintain stable model training under these conditions.

\begin{table}[H]
\centering
\caption{Results on BERT-Large, GPT-2-Large and GPT-2-XL}
\begin{tabular}{@{}ccccc@{}}
\toprule
\multirow{2}{*}{\diagbox{Method}{Hyper-params}} &
\multicolumn{2}{c}{Batch Size} & \multicolumn{2}{c}{LR Warmup Steps} \\ 
\cmidrule(l){2-3} \cmidrule(l){4-5}
& 64 & 1 & 1000 & 1\\ \midrule
Naive & $0/3_{3.0\pm0.2}$ & $4/4_{7.5\pm0.1}$ & $0/3_{3.0\pm0.2}$ & $3/3_{7.5\pm0.1}$ \\ \midrule
GC & $-$ & $0/1_{3.2}$ & $-$ & $0/1_{3.0}$ \\
OR & $-$ & $0/2_{4.0\pm0.1}$ & $-$ & $1/1_{-}$ \\
QK-Norm & $-$ & $0/2_{3.3\pm0.1}$ & $-$ & $2/2_{-}$ \\
PSS  Naive & $0/3_{3.0\pm0.1}$ & $0/4_{3.4\pm0.2}$ & $0/3_{3.0\pm0.2}$ & $0/3_{2.9\pm0.1}$ \\ \bottomrule
\end{tabular}
\label{table:bs-lrwarmup}
\end{table}

\subsection{Experiment Results on PSS Activation Frequency}
\label{appendix:PSS activation freq}

\begin{table}[H]
\centering
\caption{Frequency of PSS activation across different settings on BERT.}
\begin{tabular}{@{}ccccc@{}}
\toprule
\multirow{2}{*}{Learning Rate} &
\multicolumn{2}{c}{BERT-base} & \multicolumn{2}{c}{BERT-large} \\ 
\cmidrule(l){2-3} \cmidrule(l){4-5}
& @50k step & @100k step & @50k step & @100k step \\ \midrule
Small LR & $0.02\%$(lr=1e-4) & $0.03\%$(lr=1e-4) & $0.03\%$(lr=5e-5) & $0.03\%$(lr=5e-5) \\ 
Large LR & $0.09\%$(lr=4e-4) & $0.06\%$(lr=4e-4) & $0.16\%$(lr=1e-4) & $0.11\%$(lr=1e-4) \\
\bottomrule
\end{tabular}
\label{table:pss-activation-freq-bert}
\end{table}

\begin{table}[H]
\centering
\caption{Frequency of PSS activation across different settings on GPT-2.}
\begin{tabular}{@{}ccccc@{}}
\toprule
\multirow{2}{*}{Learning Rate} &
\multicolumn{2}{c}{GPT-2-medium} & \multicolumn{2}{c}{GPT-2-large} \\ 
\cmidrule(l){2-3} \cmidrule(l){4-5}
& @100k step & @200k step & @100k step & @200k step \\ \midrule
Small LR & $0.05\%$(lr=5e-4) & $0.03\%$(lr=5e-4) & $0.06\%$(lr=5e-5) & $0.05\%$(lr=5e-5) \\ 
Large LR & $0.23\%$(lr=1e-3) & $0.15\%$(lr=1e-3) & $0.19\%$(lr=1e-4) & $0.15\%$(lr=1e-4) \\
 \bottomrule
\end{tabular}
\label{table:pss-activation-freq-gpt}
\end{table}

\section{Empirical Observation}
\label{appendix:empirical observation}

\subsection{Observation on Representation Dominance}
\label{appendix:representation-dom}
\begin{figure}[H]
  \centering
    \includegraphics[width=.7\textwidth]{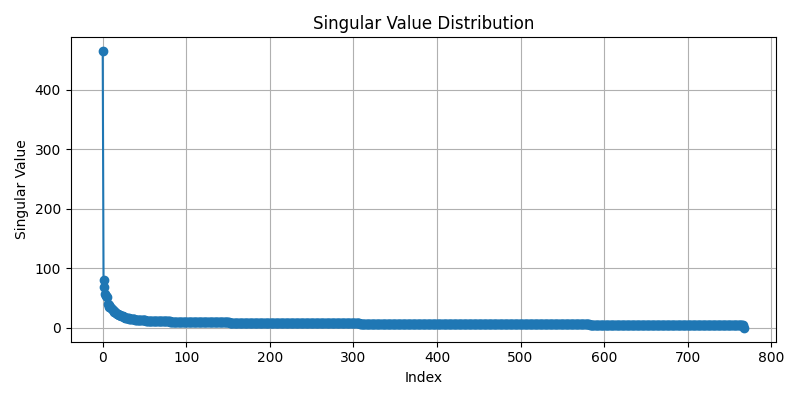}   
  \caption{Singular value distribution of the representation matrix in the embedding layer.}
\end{figure}

\subsection{Observation on Different Modules}
Below is the evolution of \( \text{SR}(\mathbf{W}) \), \( \text{SR}(\mathbf{Z}) \), singularity alignment \( \phi \), and the combined representation of training loss and gradient norm across the Attention Key, Query, value matrix, and FFN Dense 1 module.
\begin{figure}[H]
  \centering
    \includegraphics[width=\textwidth]{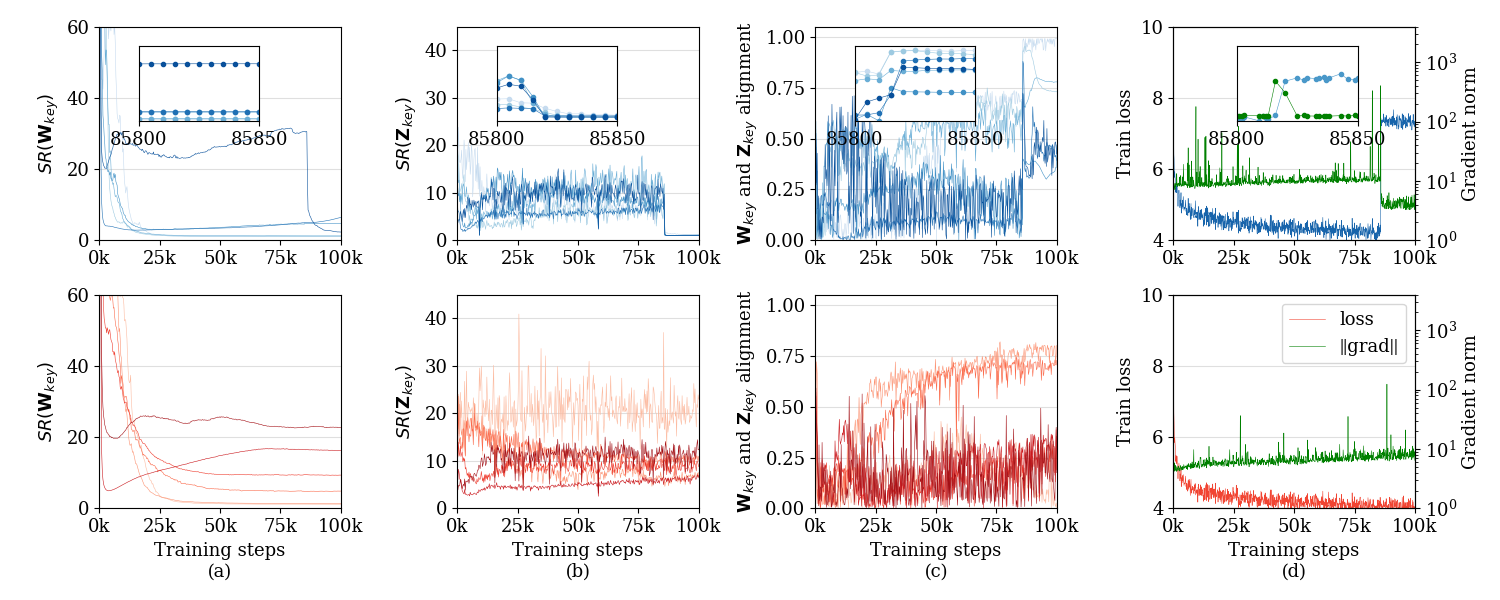}   
  \caption{Observation on Attention Key matrix.}
\end{figure}

\begin{figure}[H]
  \centering
    \includegraphics[width=\textwidth]{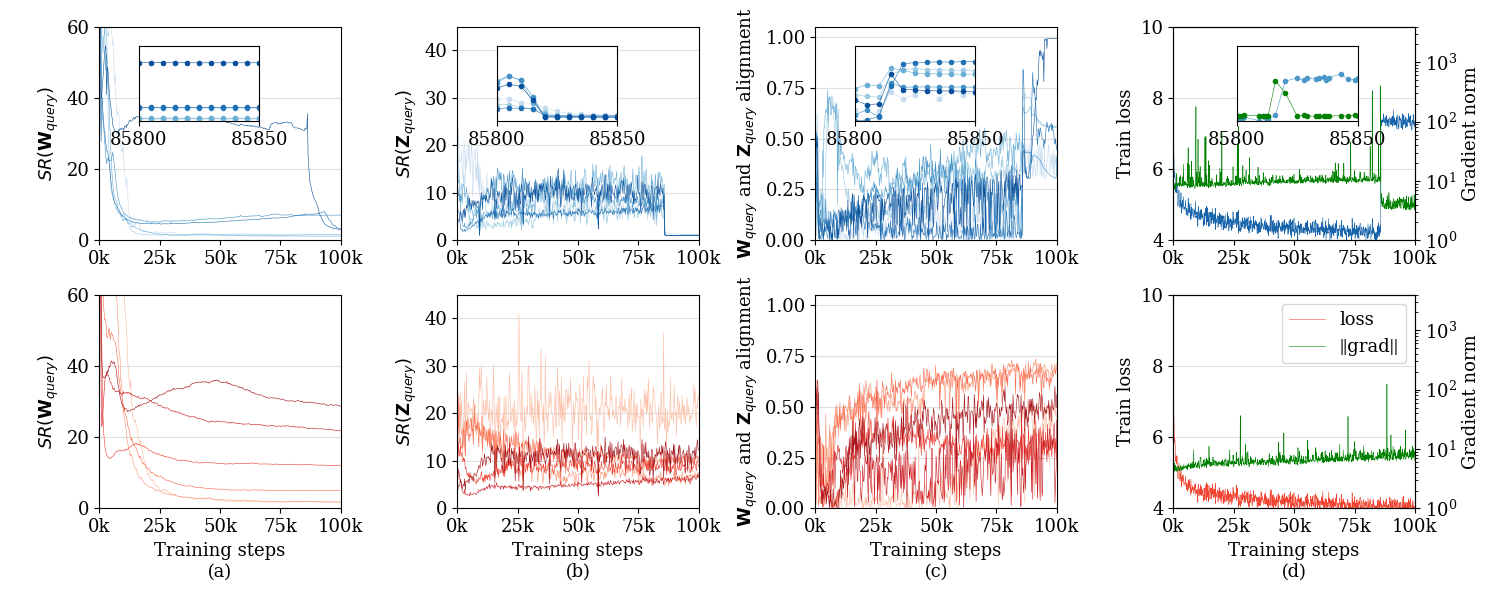}   
  \caption{Observation on Attention Query matrix.}
\end{figure}

\begin{figure}[H]
  \centering
    \includegraphics[width=\textwidth]{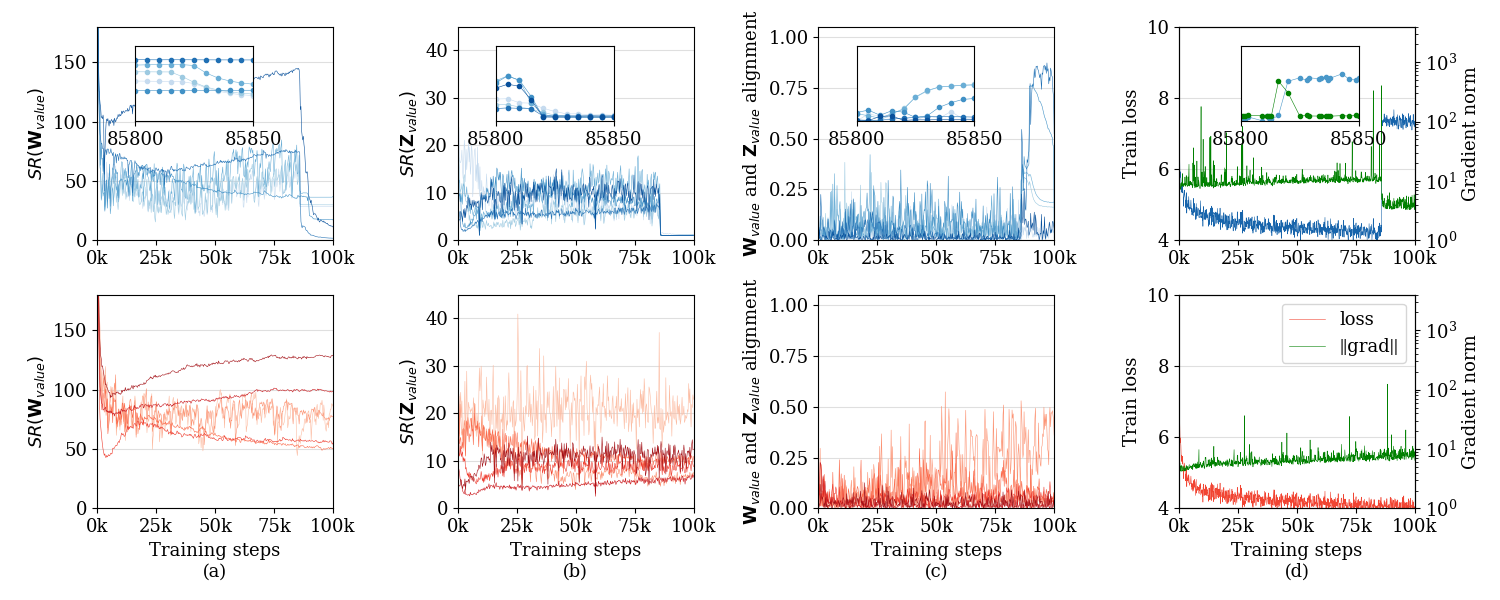}   
  \caption{Observation on Attention Value matrix.}
\end{figure}

\begin{figure}[H]
  \centering
    \includegraphics[width=\textwidth]{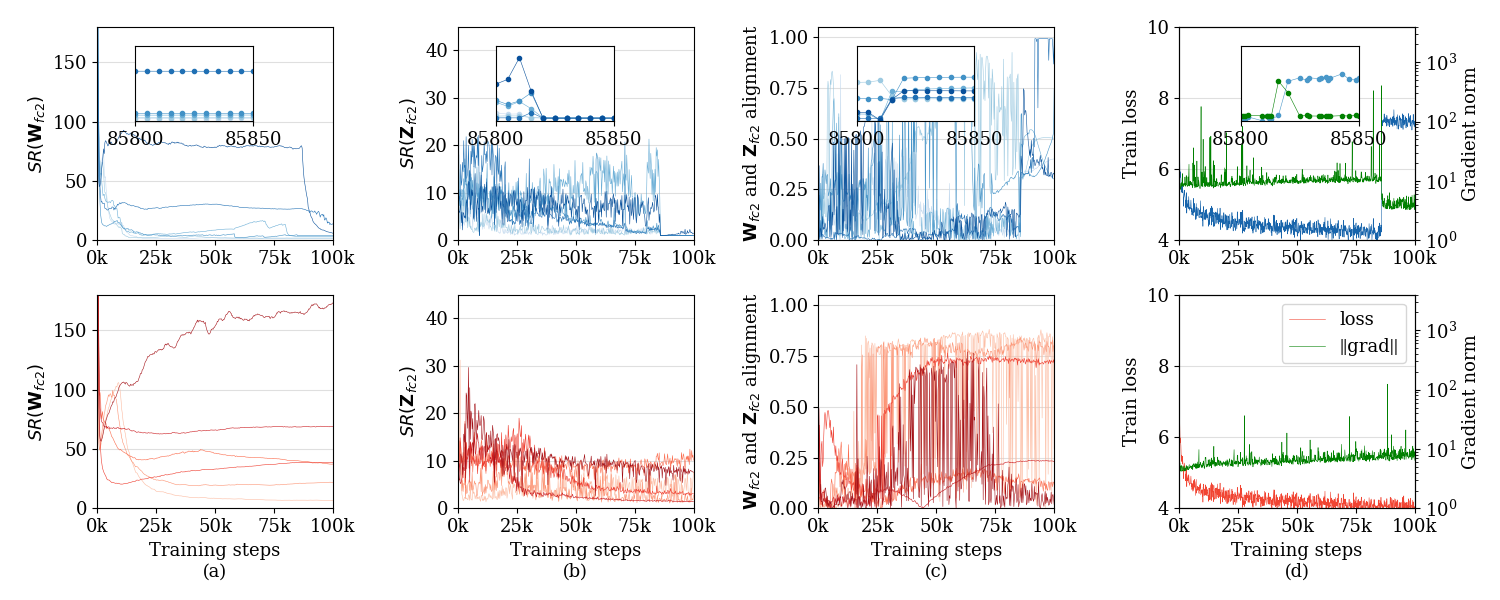}   
  \caption{Observation on FFN Dense 1 module.}
\end{figure}

\section{Limitation Discussion}
\label{appendix:limitation}

A key limitation is the absence of validation on industrial-scale models and proprietary datasets.

%% file: sections/checklist.tex
\begin{enumerate}

\item {\bf Claims}
    \item[] Question: Do the main claims made in the abstract and introduction accurately reflect the paper's contributions and scope?
    \item[] Answer: \answerYes{}
    \item[] Justification: All the claims made in abstract and introduction reflect the paper's contributions and insight clearly.

\item {\bf Limitations}
    \item[] Question: Does the paper discuss the limitations of the work performed by the authors?
    \item[] Answer: \answerYes{} 
    \item[] Justification: See appendix \ref{appendix:limitation}.

\item {\bf Theory assumptions and proofs}
    \item[] Question: For each theoretical result, does the paper provide the full set of assumptions and a complete (and correct) proof?
    \item[] Answer: \answerYes{} 
    \item[] Justification: See analysis \ref{analysis} and theoretical proof \ref{appendix:theoretical proof}.

    \item {\bf Experimental result reproducibility}
    \item[] Question: Does the paper fully disclose all the information needed to reproduce the main experimental results of the paper to the extent that it affects the main claims and/or conclusions of the paper (regardless of whether the code and data are provided or not)?
    \item[] Answer: \answerYes{} 
    \item[] Justification: We fully demonstrate our method and experiment settings in section \ref{method} and section \ref{experiments:experimental setup}. See more details in appendix \ref{appendix:smoothing-policy} and appendix \ref{appendix:model-configuration}

\item {\bf Open access to data and code}
    \item[] Question: Does the paper provide open access to the data and code, with sufficient instructions to faithfully reproduce the main experimental results, as described in supplemental material?
    \item[] Answer: \answerNo{} 
    \item[] Justification: We will fully release our code if this paper was accepted.

\item {\bf Experimental setting/details}
    \item[] Question: Does the paper specify all the training and test details (e.g., data splits, hyperparameters, how they were chosen, type of optimizer, etc.) necessary to understand the results?
    \item[] Answer: \answerYes{} 
    \item[] Justification: See \ref{experiments:experimental setup}

\item {\bf Experiment statistical significance}
    \item[] Question: Does the paper report error bars suitably and correctly defined or other appropriate information about the statistical significance of the experiments?
    \item[] Answer: \answerYes{} 
    \item[] Justification: See table \ref{table:bert-stability} and table \ref{tab:computation-cost} for example.

\item {\bf Experiments compute resources}
    \item[] Question: For each experiment, does the paper provide sufficient information on the computer resources (type of compute workers, memory, time of execution) needed to reproduce the experiments?
    \item[] Answer: \answerYes{} 
    \item[] Justification: All the experiment results are conducted with 8 A100 GPUs.
    
\item {\bf Code of ethics}
    \item[] Question: Does the research conducted in the paper conform, in every respect, with the NeurIPS Code of Ethics \url{https://neurips.cc/public/EthicsGuidelines}?
    \item[] Answer: \answerYes{} 
    \item[] Justification: Yes, we claim the research conducted in the paper conform, in every respect, with the NeurIPS Code of Ethics.

\item {\bf Broader impacts}
    \item[] Question: Does the paper discuss both potential positive societal impacts and negative societal impacts of the work performed?
    \item[] Answer: \answerNA{} 
    \item[] Justification: The paper does not appear to discuss potential positive or negative societal impacts of the work, as it emphasizes a technical advancement without addressing broader societal implications.
    
\item {\bf Safeguards}
    \item[] Question: Does the paper describe safeguards that have been put in place for responsible release of data or models that have a high risk for misuse (e.g., pretrained language models, image generators, or scraped datasets)?
    \item[] Answer: \answerNA{} 
    \item[] Justification: The paper does not release data or models that could pose a high risk for misuse.

\item {\bf Licenses for existing assets}
    \item[] Question: Are the creators or original owners of assets (e.g., code, data, models), used in the paper, properly credited and are the license and terms of use explicitly mentioned and properly respected?
    \item[] Answer: \answerYes{} 
    \item[] Justification: We have cited the original owners of assets in this paper.

\item {\bf New assets}
    \item[] Question: Are new assets introduced in the paper well documented and is the documentation provided alongside the assets?
    \item[] Answer: \answerNA{} 
    \item[] Justification: The paper does not release new assets.

\item {\bf Crowdsourcing and research with human subjects}
    \item[] Question: For crowdsourcing experiments and research with human subjects, does the paper include the full text of instructions given to participants and screenshots, if applicable, as well as details about compensation (if any)? 
    \item[] Answer: \answerNA{} 
    \item[] Justification: The paper does not involve crowdsourcing nor research with human subjects.

\item {\bf Institutional review board (IRB) approvals or equivalent for research with human subjects}
    \item[] Question: Does the paper describe potential risks incurred by study participants, whether such risks were disclosed to the subjects, and whether Institutional Review Board (IRB) approvals (or an equivalent approval/review based on the requirements of your country or institution) were obtained?
    \item[] Answer: \answerNA{} 
    \item[] Justification: The paper does not involve crowdsourcing nor research with human subjects.

\item {\bf Declaration of LLM usage}
    \item[] Question: Does the paper describe the usage of LLMs if it is an important, original, or non-standard component of the core methods in this research? Note that if the LLM is used only for writing, editing, or formatting purposes and does not impact the core methodology, scientific rigorousness, or originality of the research, declaration is not required.
    \item[] Answer: \answerNA{} 
    \item[] Justification: The core method development in this research does not involve LLMs as any important, original, or non-standard components.

\end{enumerate}